\documentclass{article}

    \PassOptionsToPackage{numbers, compress}{natbib}

\usepackage[final]{neurips_2020}

\usepackage[utf8]{inputenc} 
\usepackage[T1]{fontenc}    
\usepackage{hyperref}       
\usepackage{url}            
\usepackage{booktabs}       
\usepackage{amsfonts}       
\usepackage{nicefrac}       
\usepackage{microtype}      
\usepackage{algorithm,algorithmic}
\usepackage[utf8]{inputenc} 
\usepackage[T1]{fontenc}    
\usepackage{hyperref}       
\usepackage{url}            
\usepackage{booktabs}       
\usepackage{amsfonts}       
\usepackage{nicefrac}       
\usepackage{microtype}      
\usepackage{graphicx}
\usepackage{array}

\usepackage{hyperref}
\hypersetup{
	colorlinks   = true, 
	urlcolor     = blue, 
	linkcolor    = blue, 
	citecolor   = blue 
}

\usepackage{amsmath,amscd,amsbsy,amssymb,latexsym,url,bm,amsthm}
\allowdisplaybreaks
\usepackage{cleveref}
\usepackage{verbatim}
\usepackage{color}
\usepackage{dsfont}

\newtheorem{theorem}{Theorem}

\newtheorem{lemma}{Lemma}
\newtheorem{definition}{Definition}

\newcommand{\RR}{\mathbb{R}}
\newcommand{\cA}{\mathcal{A}}
\newcommand{\cC}{\mathcal{C}}
\newcommand{\cE}{\mathcal{E}}
\newcommand{\cF}{\mathcal{F}}

\newcommand{\cS}{\mathcal{S}}
\newcommand{\cN}{\mathcal{N}}
\newcommand{\cU}{\mathcal{U}}
\newcommand{\argmax}{\mathrm{argmax}}

\newcommand{\opt}{\mathrm{Opt}}
\newcommand{\popt}{\mathrm{POpt}}
\newcommand{\trace}{\mathrm{trace}}

\newcommand{\abs}[1]{\left| #1 \right|}
\newcommand{\bOne}[1]{\mathds{1} \! \left\{#1\right\}}
\newcommand{\bracket}[1]{\left(#1\right)}
\newcommand{\norm}[1]{\left\| #1 \right\|}
\newcommand{\set}[1]{\left\{ #1 \right\}}
\newcommand{\EE}[1]{\mathbb{E} \left[#1\right]}
\newcommand{\PP}[1]{\mathbb{P} \left(#1\right)}

\newcommand{\etal}{\emph{et al.}}
\newcommand{\bw}{\mathbf{w}}

\mathchardef\mhyphen="2D
\newcommand{\oimetc}{{\tt OIM\mhyphen ETC}}
\newcommand{\ltlinucb}{{\tt LT\mhyphen LinUCB}}
\newcommand{\oracle}{{\tt Oracle}}
\newcommand{\pairoracle}{{\tt PairOracle}}
\newif\ifsup\supfalse
\suptrue

\usepackage{tikz}
\usetikzlibrary{shapes,shapes.misc,positioning}
\usetikzlibrary{patterns,arrows,decorations.pathreplacing}

\title{Online Influence Maximization under Linear Threshold Model}

\author{
    Shuai Li$^1$\thanks{Corresponding author} \qquad Fang Kong$^1$ \qquad Kejie Tang$^1$ \qquad Qizhi Li$^1$ \qquad Wei Chen$^2$ \\ 
    $^1$Shanghai Jiao Tong University \qquad $^2$Microsoft Research  \\
    {\small \texttt{\{shuaili8,fangkong,tangkj00,qizhili\}@sjtu.edu.cn}\hskip1.8em
    \texttt{weic@microsoft.com}
    }
}

\begin{document}

\maketitle

\begin{abstract}
Online influence maximization (OIM) is a popular problem in social networks to learn influence propagation model parameters and maximize the influence spread at the same time. Most previous studies focus on the independent cascade (IC) model under the edge-level feedback. In this paper, we address OIM in the linear threshold (LT) model. Because node activations in the LT model are due to the aggregated effect of all active neighbors, it is more natural to model OIM with the node-level feedback. And this brings new challenge in online learning since we only observe aggregated effect from groups of nodes and the groups are also random. 
Based on the linear structure in node activations, we incorporate ideas from linear bandits and design an algorithm $\ltlinucb$ that is consistent with the observed feedback. 
By proving group observation modulated (GOM) bounded smoothness property, a novel result of the influence difference in terms of the random observations, we provide a regret of order $\tilde{O}(\mathrm{poly}(m)\sqrt{T})$, where $m$ is the number of edges and $T$ is the number of rounds. 
This is the first theoretical result in such order for OIM under the LT model. 
In the end, we also provide an algorithm $\oimetc$ with regret bound $O(\mathrm{poly}(m)\ T^{2/3})$, which is model-independent, simple and has less requirement on online feedback and offline computation.
\end{abstract}
  

\section{Introduction}
\label{sec:introduction}

Social networks play an important role in spreading information in people's life. In viral marketing, companies wish to broadcast their products by making use of the network structure and characteristics of influence propagation. Specifically, they want to provide free products to the selected users (seed nodes), let them advertise through the network and maximize the purchase. There is a budget of the free products and the goal of the companies is to select the optimal seed set to maximize the influence spread. This problem is called influence maximization (IM) \cite{Kempe2003} and has a wide range of applications including recommendation systems, link prediction and information diffusion. 

In the IM problem, the social network is usually modeled as a directed graph with nodes representing users and directed edges representing influence relationship between users. IM studies how to select a seed set under a given influence propagation model to maximize the influence spread when the weights are known. Independent cascade (IC) model and linear threshold (LT) model \cite{Kempe2003} are two most widely used models to characterize the influence propagation in a social network, and both models use weights on edges as model parameters.

In many real applications, however, the weights are usually unknown in advance. 
For example, in viral marketing, it is unrealistic to assume that the companies know the influence abilities beforehand. A possible solution is to learn those parameters from the diffusion data collected in the past \cite{LearnFromPast1,LearnFromPast3}.
But this method lacks the ability of adaptive learning based on the need of influence maximization. 
This motivates the studies on the online influence maximization (OIM) problem \cite{SiyuLei2015,WeiChen2013,WeiChen2016,WeiChen2017,zhengwen2017nips,IMFB2019,Sharan2015-nodelevel,Model-Independent2017}, where the learner tries to estimate model parameters and maximize influence in an iterative manner.

The studies on OIM are based on the multi-armed bandit (MAB) problem, which is a classical online learning framework and has been well studied in the literature \cite{lattimorebandit}. 
MAB problem is formulated as a $T$-round game between a learner and the environment. 
In each round, the learner needs to decide which action to play and the environment will then reveal a reward according to the chosen action. The objective of the learner is to accumulate as many rewards as possible. An MAB algorithm needs to deal with the tradeoff between exploration and exploitation: whether the learner should try actions that has not been explored well yet (exploration) or focus on the action with the best performance so far (exploitation). 
Two algorithms, the explore-then-commit (ETC) \cite{garivier2016explore} and the upper confidence bound (UCB) \cite{auer2002finite}, are widely followed in the stochastic MAB setting, where the reward of each action follows an unknown but fixed distribution.

Most existing works in OIM focus on IC model under edge-level feedback \cite{WeiChen2013,WeiChen2016,WeiChen2017,zhengwen2017nips,IMFB2019}, where the information propagates independently between pairs of users and the learner can observe the liveness of individual edges as long as its source node is influenced. The independence assumption makes the formulation simple but a bit unrealistic. 
Often in the real scenarios, the influence propagations are correlated with each other. 
The LT model is usually used to model the herd behavior that a person is more likely to be influenced if many of her friends are influenced \cite{ThresholdModel2007,ThresholdModel1978,khalil2014scalable}. 
Thus for the LT model, it is more natural to use the node-level feedback where we only observe the node activations, since it is hard to pinpoint which neighbor or neighbors actually contribute to an activation in a herd behavior.

In this paper, we first formulate the OIM problem under the LT model with the node-level feedback and distill effective information based on the feedback. 
The main challenge is that only the aggregated group effect on node activations can be observed and the aggregated groups are also random. 
Based on the linear structure of the LT model, we incorporate the idea of linear bandits and propose the $\ltlinucb$ algorithm, whose update mechanism is consistent with the distilled information.
By proving group observation modulated (GOM) bounded smoothness, a key property on the influence spread under two different weight vectors, we can bound the regret. 
Such a property is similar to the triggering probability modulated (TPM) bounded smoothness condition under the IC model with edge-level feedback \cite{WeiChen2017}, but the derivation in our case under the node-level feedback is more difficult. 
The regret is of order $O(\mathrm{poly}(m) \sqrt{T} \log T)$, where $m$ is the number of edges and $T$ is the number of rounds. 
Our $\ltlinucb$ is the first OIM algorithm under the LT model that achieves the regret in this order. 
Finally we give $\oimetc$ algorithm, applying to both IC and LT with node-level feedback. Though simple, it has less requirement on the observed feedback and the offline computation, and it achieves the regret bound $O(\mathrm{poly}(m) T^{2/3}), O(\mathrm{poly}(m) \log(T) / \Delta^2)$.


\paragraph{Related Work}

The problem of IM was first proposed as a discrete optimization problem by Kempe \etal~\cite{Kempe2003}.
Since then, various aspects of IM have been extensively studied (see~\cite{Wei2013Information,LiFWT18} for surveys in this area).
Two most popular models in this field are the IC and LT models. 
The former assumes that the influence between pairs of users are independent and the latter characterizes the herd behavior. Some works \cite{wang2012scalableIC,jung2012irieIC,Kempe2003,IMM2015} study the IC model and some  \cite{chen2010scalableLT,goyal2011simpathLT,Kempe2003,IMM2015} study the LT model. They all assume the weights on the edges are known and focus on the model properties and approximated solutions. We treat them as the \textit{offline} setting.

When the weight vectors are unknown, Chen \etal~\cite{WeiChen2016,WeiChen2017} study the problem in the \textit{online} setting, selecting seed sets as well as learning the parameters.
They study the IC model with edge-level feedback, propose CUCB algorithm and show that CUCB achieves the distribution-dependent and distribution-independent regret bounds of $O(\mathrm{poly}(m)\log(T))$ and $O(\mathrm{poly}(m)\sqrt{T})$ respectively. 
Later Wen \etal~\cite{zhengwen2017nips} consider the large-scale setting and assume the edge probability is a linear function of the edge's feature vector. 
They provide a LinUCB-based algorithm with $O(dmn\sqrt{T}\ln(T))$ worst-case regret, where $d$ is the feature dimension and $n$ is the number of nodes.
Wu \etal~\cite{IMFB2019} assume that each edge probability can be decomposed as the product of the influence probability of the start node and the susceptibility probability of the end node motivated by network assortativity.
All these works study the IC model with edge-level feedback.

Vaswani \etal~\cite{Model-Independent2017} uses a heuristic objective function for OIM and brings up a model-independent algorithm under the pairwise feedback, where a node is influenced by a seed node or not. This applies to both IC and LT and the feedback scheme is relaxed than the edge-level feedback. Unfortunately, however, the heuristic objective has no theoretical approximation guarantee. 
Also, Vaswani \etal~\cite{Sharan2015-nodelevel} study the IC model with node-level feedback about the estimation gap to that under the edge-level feedback but has no regret analysis. A report \cite{vaswanilearningLT} studies the LT model with node-level feedback by optimization approaches but without theoretical guarantees.
There is another work \cite{lattimore2015linear} studying the problem of linear multi-resource allocation, which can be formulated as a bipartite LT model. But they assume every node in the left partition (resources) is selected and the algorithm needs to assign allocations for each pair of left node and right node (tasks) representing the corresponding allocation of resources on tasks. Thus the problem is different from our OIM.
The OIM problem under LT has been open for several years. We are the first to provide a reasonable formulation with an algorithm of regret $\tilde{O}(\sqrt{T})$.

OIM is a variant of combinatorial MAB (CMAB) \cite{WeiChen2013,kong2020survey}, where in each round the learner selects a combination of (base) arms. Most works \cite{kveton2015tightcombinatorial, kveton2014matroid} study stochastic setting with the linear objective and semi-bandit feedback where the learner can observe the selected base arm's reward and the reward of the action is a linear function of these base arms' rewards. 
CMAB in the stochastic setting with the linear objective and bandit feedback, where only the linear reward of the selected combination can be observed, is a special case of linear bandits. In the linear bandit setting, the learner selects a vector each round and the reward is a linear function of the selected vector action \cite{auer2002using}. The most popular method to solve it is to construct confidence ellipsoids \cite{dani2008stochastic,Linearbandits,rusmevichientong2010linearly}. 
There are also works \cite{cesa2012combinatorialbanditfeedback,combes2015combinatorialfeedback} for CMAB in the adversarial setting and bandit feedback. 
But OIM is different: its objective function is non-linear and is dependent on unchosen and probabilistically triggered base arms.

OIM is related to the problem of online learning with graph feedback \cite{Alon15} where the learner can observe the feedback of unchosen arms based on the graph structure.
Though some of them study random graphs \cite{liu2018information,li2020stochastic,kocak2016online} where the set of observed arms is random, the settings are different. 
Under the graph feedback, the observations of unchosen arms are additional and the reward only depends on the chosen arms, while under the OIM, the additional observations also contribute to the reward. Cascading bandits \cite{kveton2015cascading,li2016contextual} also consider triggering on any selected list of arms and the triggering is in the order of the lists. Compared with graph feedback and OIM setting, its triggering graph is determined by the learning agent, not the adversary.

As a generalization of graph feedback, partial monitoring \cite{bartok2014partial} is also related to OIM. Most works in this direction, if applied directly to the OIM setting, are inefficient due to the exponentially large action space. 
Lin \etal~\cite{lin2014combinatorialpm} study a combinatorial version of partial monitoring and their algorithm provides a regret of order $O(\mathrm{poly}(m)T^{2/3}\log T)$ for OIM with LT. Our $\oimetc$, however, has regret bounds of $O(\mathrm{poly}(m)T^{2/3})$ (better in the order of $T$) as well as a problem-dependent bound $O(\mathrm{poly}(m)\log T)$.


\section{Setting}
\label{sec:setting}

This section describes the setting of online influence maximization (OIM) under linear threshold (LT) model. The IM problem characterizes how to choose the seed nodes to maximize the influence spread on a social network. The network is usually represented by a directed graph $G=(V,E)$ where $V$ is the set of users and $E$ is the set of relationships between users. Each edge $e$ is associated with a weight $w(e) \in [0,1]$. For example, an edge $e=(u,v)=:e_{u,v}$ could represent user $v$ follows user $u$ on Twitter and $w(e)$ represents the `influence ability' of user $u$ on user $v$. Denote $w=(w(e))_{e\in E}$ to be the weight vector and $n=\abs{V}, m=\abs{E}, D$ to be node number, edge number and the propagation diameter respectively, where the propagation diameter is defined as the length of the longest simple path in the graph. Let $N(v)=N^{\mathrm{in}}(v)$ be the set of all incoming neighbors of $v$, shortened as in-neighbors. 

Recall that under IC model, each edge is alive with probability equal to the associated weight independently and a node is influenced if there is a directed path connecting from a seed node in the realized graph. Compared to the IC model, the LT model does not require the strong assumption of independence and describes the joint influence of the active in-neighbors on a user, reflecting the herd behavior that often occurs in real life \cite{ThresholdModel2007,ThresholdModel1978,khalil2014scalable}.

Now we describe in detail the diffusion process under the LT model. Suppose the seed set is $S$. In the beginning, each node is assigned with a threshold $\theta_v$, which is independently uniformly drawn from $[0,1]$ and characterizes the susceptibility level of node $v$. Denote $\theta=(\theta_v)_{v\in V}$ to be the threshold vector. Let $S_\tau$ be the set of activated nodes by the end of time $\tau$. At time $\tau=0$, only nodes in the seed set are activated: $S_0=S$. At time $\tau+1$ with $\tau \ge 0$, for any node $v \notin S_{\tau}$ that has not been activated yet, it will be activated if the aggregated influence of its active in-neighbors exceeds its threshold:
$
\sum\limits_{u \in N(v)\cap S_{\tau}} w(e_{u,v}) \geq \theta_v
$.
Such diffusion process will last at most $D$ time steps. 
The size of the influenced nodes is denoted as $r(S,w,\theta) = \abs{S_{D}}$. Let $r(S, w) = \EE{r(S, w, \theta)}$ be the \textit{influence spread} of seed set $S$ where the expectation is taken over all random variables $\theta_v$'s. The IM problem is to find the seed set $S$ with the size at most $K$ under weight vector $w$ to maximize the influence spread,
$
\max_{S \in \cA} r(S,w)
$,
where $\cA = \set{S \subset V: \abs{S} \le K}$ is the \textit{action set} for the seed nodes. 
We also adopt the usual assumption that $\sum_{u \in N(v)}w(e_{u,v}) \le 1$ for any $v \in V$. This assumption makes LT have an equivalent live-edge graph formulation like IC model \cite{Kempe2003,Wei2013Information}.
The term of graph $G$ and seed size $K$ will be omitted when the context is clear. Here we emphasize that the model parameters are the weights $w$ while the threshold vector $\theta$ is not model parameter (which follows uniform distribution).

The (offline) IM is NP-hard under the LT model but it can be approximately solved \cite{Kempe2003,IMM2015}. For a fixed weight vector $w$, let $S_w^{\opt}$ be an optimal seed set and $\opt_w$ be its corresponding influence spread: $S_w^{\opt} \in \argmax_{S \in \cA} r(S,w)$ and $\opt_w = r(S_w^{\opt},w)$. Let \oracle\ be an (offline) oracle that outputs a solution given the weight vector as input. Then for $\alpha,\beta\in [0,1]$, the \oracle \ is an $(\alpha,\beta)$-approximation if $\PP{r(S',w) \geq \alpha \cdot \opt_w} \geq \beta$ where $S'=\oracle(w)$ is a solution returned by the \oracle \ for the weight vector $w$. Note when $\alpha=\beta=1$ the oracle is \textit{exact}.

The online version is to maximize the influence spread when the weight vector (or the model parameter) $w=(w(e))_{e\in E}$ is unknown. In each round $t$, the learner selects a seed set $S_t$, receives the observations and then updates itself accordingly. 
For the type of observations, previous works on IC mostly assume the edge-level feedback: the learner can observe the outgoing edges of each active node \cite{WeiChen2016,zhengwen2017nips,IMFB2019}. 
But for the LT model, it is not very realistic to assume the learner can observe which in-neighbor influences the target user since the LT model characterizes the aggregate influence of a crowd. So we consider a more realistic \textit{node-level feedback}\footnote{One may think of the node-level feedback as knowing only the set of nodes activated by the end of the diffusion process. We refer to this as (partial) node-level feedback and ours as (full) node-level feedback. This naming comes from \cite{narasimhan2015learnability}.} in this paper: the learner can only observe the influence diffusion process on node sets as $S_{t,0},\ldots,S_{t,\tau},\ldots$ in round $t$.

The objective of the OIM is to minimize the cumulative $\eta$-scaled regret \cite{WeiChen2013,zhengwen2017nips} over total $T$ rounds:
\begin{align}
\label{eq:DefinitionOfRegret}
R(T) = \EE{\sum_{t=1}^{T}R_t }= \EE{\eta \cdot T \cdot \opt_w - \sum_{t=1}^{T} r(S_t, w)}\,,
\end{align}
where the expectation is over the randomness on the threshold vector and the output of the adopted offline oracle in each round .

Throughout this paper, we will use `round' $t$ to denote a step in online learning and use `time' $\tau$ of round $t$ to denote an influence diffusion step of seed set $S_t$ in round $t$.


\section{$\ltlinucb$ Algorithm}\label{sec:linucb}

In this section, we show how to distill effective information based on the feedback and propose a LinUCB-type algorithm, $\ltlinucb$, for OIM under LT. 
For each node $v \in V$, denote $w_v = (w(e_{u,v}))_{u \in N(v)}$ to be the weight vector of its incoming edges. Let $\chi(e_{u,v}) \in \set{0,1}^{\abs{N(v)}}$ be the one-hot representation of the edge $e_{u,v}$ over all of $v$'s incoming edges $\set{e_{u,v}: u \in N(v)}$, that is its $e'$-entry is $1$ if and only if $e'=e_{u,v}$. Then $w(e_{u,v}) = \chi(e_{u,v})^\top w_v$. 
For a subset of edges $E' \subseteq \set{e_{u,v}: u \in N(v)}$, we define $\chi(E'):=\sum_{e\in E'} \chi(e) \in \set{0,1}^{\abs{N(v)}}$ to be the vector whose $e$-entry is $1$ if and only if $e \in E'$. Here we abuse the notation that $\chi(\set{e}) = \chi(e)$. By this notation, the weight sum of the edges in $E'$ is simply written as $\chi(E')^\top w_v$. A subset $V' \subset N(v)$ of $v$'s in-neighbors can activate $v$ if the weight sum of associated edges exceeds the threshold, that is $\chi(E')^\top w_v \ge \theta_v$ with $E'= \set{e_{u,v}: u \in V'}$.

Fix a diffusion process $S_0, S_1, \ldots, S_\tau, \ldots$, where the seed set is $S_0$. For each node $v$, define 
\begin{align}
  \tau_1(v) := \min_{\tau}\set{\tau=0,\ldots,D: N(v) \cap S_{\tau} \neq \emptyset }
\end{align}
as the earliest time step when node $v$ has active in-neighbors. Particularly we set $\tau_1(v) = D+1$ if node $v$ has no active in-neighbor until the diffusion ends. For any $\tau \ge \tau_1(v)$, further define 
\begin{align}
  E_{\tau}(v) :=\set{e_{u,v} : u \in N(v)\cap S_{\tau}}
\end{align}
as the set of incoming edges associated with active in-neighbors of $v$ at time step $\tau$. 

Recall that the learner can only observe the aggregated influence ability of a node's active in-neighbors. 
Let $\tau_2(v)$ represent the time step that node $v$ is influenced ($\tau_2(v)> \tau_1(v)$), which is equivalent to mean that $v$'s active in-neighbors of time $\tau_2(v)-1$ succeed to influence it but those in time $\tau_2(v)-2$ fail ($E_{-1} = \emptyset$). Thus the defintion of $\tau_2(v)$ can be written as
\begin{align}
    \tau_2(v):= \set{\tau=0,\ldots,D: \chi(E_{\tau-2}(v))^\top w_v < \theta_v \le \chi(E_{\tau-1}(v))^\top w_v } \,.
\end{align}
For consistency, we set $\tau_2(v) = D+1$ if node $v$ is finally not influenced after the information diffusion ends. 
Then based on the definition of $\tau_1(v)$ and $\tau_2(v)$, we can obtain that node $v$ is not influenced at time $\tau \in (\tau_1(v), \tau_2(v))$, which means that the set of active in-neighbors of $v$ at time step $\tau-1$ fails to activate it.  

According to the rule of information diffusion under the LT model, an event that $E' \subseteq \set{e_{u,v}: u \in N(v)}$ fails to activate $v$ is equivalent to $\chi(E')^\top w_v < \theta_v$, which happens with probability $1 - \chi(E')^\top w_v$ since $\theta_v$ is uniformly drawn from $[0,1]$. Similarly an event that $E' \subseteq \set{e_{u,v}: u \in N(v)}$ succeeds to activate $v$ is equivalent to $\chi(E')^\top w_v \ge \theta_v$, which happens with probability $\chi(E')^\top w_v$. 
So for node $v$ who has active in-neighbors, $v$ is not influenced at time step $\tau$ ($\tau_1(v)<\tau<\tau_2(v)$) means that the set of $v$'s active in-neighbors by $\tau-1$ fails to activate it, thus we can use $(\chi(E_{\tau-1}(v)), 0)$ to update our belief on the unknown weight vector $w_v$; $v$ is influenced at time step $\tau_2(v)$ means that the set of $v$'s active in-neighbors by $\tau_2(v)-1$ succeeds to activate it, we can thus use $(\chi(E_{\tau_2(v)-1}(v)), 1)$ to update our belief on the unknown weight vector $w_v$; 
$v$ is finally not influenced means that all of its active in-neighbors (by time step $D$) fail to activate it, we can use $(\chi(E_{\tau_2(v)-1}(v)), 0)$ to update $w_v$ since $\tau_2(v)$ is defined as $D+1$ in this case. 
Note all of these events are correlated (based on a same $\theta_v$), thus we can only choose at most one of them to update $w_v$ for node $v$ who has active in-neighbors. 
If $v$ has no active in-neighbors, we have no observation on $w_v$ and could update nothing. 

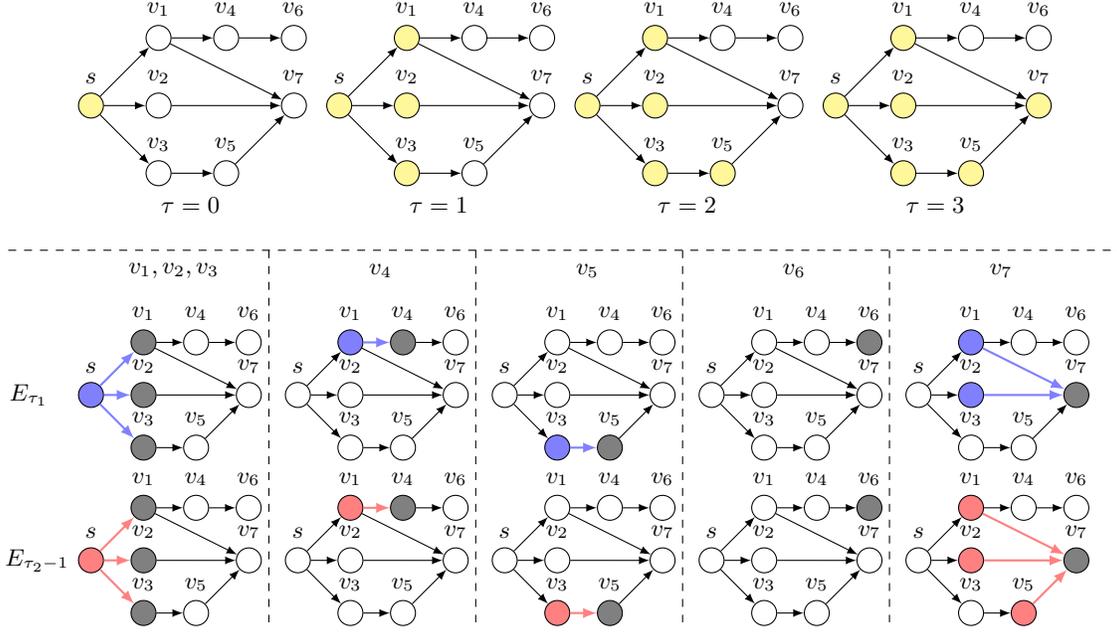
\begin{figure}[tbh!]
\hspace{-0.6cm}
    \begin{tikzpicture}[scale=0.55,font=\small,every node/.style={node distance = 0.9cm}]

    \node (s)  [circle,draw=black,fill=yellow!50,label=$s$] at (-24,-1) {} ;
    \node (v2) [circle,draw=black,right of=s,,label=$v_2$] {};
  \node (v1) [circle,draw=black,above of=v2,label=$v_1$]{};
  \node (v3) [circle,draw=black,below of=v2,label=$v_3$] {};
  \node (v4) [circle,draw=black,right of=v1,label=$v_4$] {};
  \node (v5) [circle,draw=black,right of=v3,label=$v_5$] {};
  \node (v6) [circle,draw=black,right of=v4,label=$v_6$] {};
  \node (v7) [circle,draw=black,right of=v2,node distance = 1.8cm,label=$v_7$] {};

  \node [below right of=v3,node distance = 0.6cm] {$\tau=0$};
  
  \draw [-latex](s)--(v1);
  \draw [-latex](s)--(v2);
  \draw [-latex](s)--(v3);
  \draw [-latex](v1)--(v4);
  \draw [-latex](v4)--(v6);
  \draw [-latex](v1)--(v7);
  \draw [-latex](v2)--(v7);
  \draw [-latex](v3)--(v5);
  \draw [-latex](v5)--(v7);


  \node (1s)  [circle,draw=black,fill=yellow!50,label=$s$] at (-18,-1) {} ;
    \node (1v2) [circle,draw=black,right of=1s,fill=yellow!50,label=$v_2$] {};
  \node (1v1) [circle,draw=black,above of=1v2,fill=yellow!50,label=$v_1$] {};
  \node (1v3) [circle,draw=black,below of=1v2,fill=yellow!50,label=$v_3$] {};
  \node (1v4) [circle,draw=black,right of=1v1,label=$v_4$] {};
  \node (1v5) [circle,draw=black,right of=1v3,label=$v_5$] {};
  \node (1v6) [circle,draw=black,right of=1v4,label=$v_6$] {};
  \node (1v7) [circle,draw=black,right of=1v2,node distance = 1.8cm,label=$v_7$] {};

  \node [below right of=1v3,node distance = 0.6cm] {$\tau=1$};
  
  \draw [-latex](1s)--(1v1);
  \draw [-latex](1s)--(1v2);
  \draw [-latex](1s)--(1v3);
  \draw [-latex](1v1)--(1v4);
  \draw [-latex](1v4)--(1v6);
  \draw [-latex](1v1)--(1v7);
  \draw [-latex](1v2)--(1v7);
  \draw [-latex](1v3)--(1v5);
  \draw [-latex](1v5)--(1v7);


  \node (2s)  [circle,draw=black,fill=yellow!50,label=$s$] at (-12,-1) {} ;
    \node (2v2) [circle,draw=black,right of=2s,fill=yellow!50,label=$v_2$] {};
  \node (2v1) [circle,draw=black,above of=2v2,fill=yellow!50,label=$v_1$] {};
  \node (2v3) [circle,draw=black,below of=2v2,fill=yellow!50,label=$v_3$] {};
  \node (2v4) [circle,draw=black,right of=2v1,label=$v_4$] {};
  \node (2v5) [circle,draw=black,right of=2v3,fill=yellow!50,label=$v_5$] {};
  \node (2v6) [circle,draw=black,right of=2v4,label=$v_6$] {};
  \node (2v7) [circle,draw=black,right of=2v2,node distance = 1.8cm,label=$v_7$] {};

  \node [below right of=2v3,node distance = 0.6cm] {$\tau=2$};
  
  \draw [-latex](2s)--(2v1);
  \draw [-latex](2s)--(2v2);
  \draw [-latex](2s)--(2v3);
  \draw [-latex](2v1)--(2v4);
  \draw [-latex](2v4)--(2v6);
  \draw [-latex](2v1)--(2v7);
  \draw [-latex](2v2)--(2v7);
  \draw [-latex](2v3)--(2v5);
  \draw [-latex](2v5)--(2v7);


  \node (3s)  [circle,draw=black,fill=yellow!50,label=$s$] at (-6,-1) {} ;
    \node (3v2) [circle,draw=black,right of=3s,fill=yellow!50,label=$v_2$] {};
  \node (3v1) [circle,draw=black,above of=3v2,fill=yellow!50,label=$v_1$] {};
  \node (3v3) [circle,draw=black,below of=3v2,fill=yellow!50,label=$v_3$] {};
  \node (3v4) [circle,draw=black,right of=3v1,label=$v_4$] {};
  \node (3v5) [circle,draw=black,right of=3v3,fill=yellow!50,label=$v_5$] {};
  \node (3v6) [circle,draw=black,right of=3v4,label=$v_6$] {};
  \node (3v7) [circle,draw=black,right of=3v2,node distance = 1.8cm,fill=yellow!50,label=$v_7$] {};

  \node [below right of=3v3,node distance = 0.6cm] {$\tau=3$};
  
  \draw [-latex](3s)--(3v1);
  \draw [-latex](3s)--(3v2);
  \draw [-latex](3s)--(3v3);
  \draw [-latex](3v1)--(3v4);
  \draw [-latex](3v4)--(3v6);
  \draw [-latex](3v1)--(3v7);
  \draw [-latex](3v2)--(3v7);
  \draw [-latex](3v3)--(3v5);
  \draw [-latex](3v5)--(3v7);


  \draw[dashed] (-26,-4.5) -- (1,-4.5);

  \draw[dashed] (-19.7,-4.5) -- (-19.7,-13.5);
  \draw[dashed] (-14.7,-4.5) -- (-14.7,-13.5);
  \draw[dashed] (-9.7,-4.5) -- (-9.7,-13.5);
  \draw[dashed] (-4.7,-4.5) -- (-4.7,-13.5);


  \node (e1) at (-25.5,-8) {$E_{\tau_1}$} ;
  \node (e2) at (-25.3,-12) {$E_{\tau_2-1}$} ;

  \node (ev1) at (-22,-5) {$v_1,v_2,v_3$} ;

  \node (se1)  [circle,draw=black,fill=blue!50,label=$s$] at (-24,-8) {} ;
    \node (v2e1) [circle,draw=black,right of=se1,node distance = 0.7cm,fill=black!50,label=$v_2$] {};
  \node (v1e1) [circle,draw=black,above of=v2e1,node distance = 0.7cm,fill=black!50,label=$v_1$]{};
  \node (v3e1) [circle,draw=black,below of=v2e1,node distance = 0.7cm,fill=black!50,label=$v_3$] {};
  \node (v4e1) [circle,draw=black,right of=v1e1,node distance = 0.7cm,label=$v_4$] {};
  \node (v5e1) [circle,draw=black,right of=v3e1,node distance = 0.7cm,label=$v_5$] {};
  \node (v6e1) [circle,draw=black,right of=v4e1,node distance = 0.7cm,label=$v_6$] {};
  \node (v7e1) [circle,draw=black,right of=v2e1,node distance = 1.4cm,label=$v_7$] {};

  \draw [-latex,thick,color=blue!50](se1)--(v1e1);
  \draw [-latex,thick,color=blue!50](se1)--(v2e1);
  \draw [-latex,thick,color=blue!50](se1)--(v3e1);
  \draw [-latex](v1e1)--(v4e1);
  \draw [-latex](v4e1)--(v6e1);
  \draw [-latex](v1e1)--(v7e1);
  \draw [-latex](v2e1)--(v7e1);
  \draw [-latex](v3e1)--(v5e1);
  \draw [-latex](v5e1)--(v7e1);

  \node (se2)  [circle,draw=black,fill=red!50,label=$s$] at (-24,-12) {} ;
    \node (v2e2) [circle,draw=black,right of=se2,node distance = 0.7cm,fill=black!50,label=$v_2$] {};
  \node (v1e2) [circle,draw=black,above of=v2e2,node distance = 0.7cm,fill=black!50,label=$v_1$]{};
  \node (v3e2) [circle,draw=black,below of=v2e2,node distance = 0.7cm,fill=black!50,label=$v_3$] {};
  \node (v4e2) [circle,draw=black,right of=v1e2,node distance = 0.7cm,label=$v_4$] {};
  \node (v5e2) [circle,draw=black,right of=v3e2,node distance = 0.7cm,label=$v_5$] {};
  \node (v6e2) [circle,draw=black,right of=v4e2,node distance = 0.7cm,label=$v_6$] {};
  \node (v7e2) [circle,draw=black,right of=v2e2,node distance = 1.4cm,label=$v_7$] {};

  \draw [-latex,thick,color=red!50](se2)--(v1e2);
  \draw [-latex,thick,color=red!50](se2)--(v2e2);
  \draw [-latex,thick,color=red!50](se2)--(v3e2);
  \draw [-latex](v1e2)--(v4e2);
  \draw [-latex](v4e2)--(v6e2);
  \draw [-latex](v1e2)--(v7e2);
  \draw [-latex](v2e2)--(v7e2);
  \draw [-latex](v3e2)--(v5e2);
  \draw [-latex](v5e2)--(v7e2);


  \node (ev4) at (-17,-5) {$v_4$} ;

  \node (se3)  [circle,draw=black,label=$s$] at (-19,-8) {} ;
    \node (v2e3) [circle,draw=black,right of=se3,node distance = 0.7cm,label=$v_2$] {};
  \node (v1e3) [circle,draw=black,above of=v2e3,fill=blue!50,node distance = 0.7cm,label=$v_1$]{};
  \node (v3e3) [circle,draw=black,below of=v2e3,node distance = 0.7cm,label=$v_3$] {};
  \node (v4e3) [circle,draw=black,right of=v1e3,node distance = 0.7cm,fill=black!50,label=$v_4$] {};
  \node (v5e3) [circle,draw=black,right of=v3e3,node distance = 0.7cm,label=$v_5$] {};
  \node (v6e3) [circle,draw=black,right of=v4e3,node distance = 0.7cm,label=$v_6$] {};
  \node (v7e3) [circle,draw=black,right of=v2e3,node distance = 1.4cm,label=$v_7$] {};

  \draw [-latex](se3)--(v1e3);
  \draw [-latex](se3)--(v2e3);
  \draw [-latex](se3)--(v3e3);
  \draw [-latex,thick,color=blue!50](v1e3)--(v4e3);
  \draw [-latex](v4e3)--(v6e3);
  \draw [-latex](v1e3)--(v7e3);
  \draw [-latex](v2e3)--(v7e3);
  \draw [-latex](v3e3)--(v5e3);
  \draw [-latex](v5e3)--(v7e3);

  \node (se4)  [circle,draw=black,label=$s$] at (-19,-12) {} ;
    \node (v2e4) [circle,draw=black,right of=se4,node distance = 0.7cm,label=$v_2$] {};
  \node (v1e4) [circle,fill=red!50,draw=black,above of=v2e4,node distance = 0.7cm,label=$v_1$]{};
  \node (v3e4) [circle,draw=black,below of=v2e4,node distance = 0.7cm,label=$v_3$] {};
  \node (v4e4) [circle,draw=black,right of=v1e4,node distance = 0.7cm,fill=black!50,label=$v_4$] {};
  \node (v5e4) [circle,draw=black,right of=v3e4,node distance = 0.7cm,label=$v_5$] {};
  \node (v6e4) [circle,draw=black,right of=v4e4,node distance = 0.7cm,label=$v_6$] {};
  \node (v7e4) [circle,draw=black,right of=v2e4,node distance = 1.4cm,label=$v_7$] {};

  \draw [-latex](se4)--(v1e4);
  \draw [-latex](se4)--(v2e4);
  \draw [-latex](se4)--(v3e4);
  \draw [-latex,thick,color=red!50](v1e4)--(v4e4);
  \draw [-latex](v4e4)--(v6e4);
  \draw [-latex](v1e4)--(v7e4);
  \draw [-latex](v2e4)--(v7e4);
  \draw [-latex](v3e4)--(v5e4);
  \draw [-latex](v5e4)--(v7e4);


  \node (ev5) at (-12,-5) {$v_5$} ;

  \node (se5)  [circle,draw=black,label=$s$] at (-14,-8) {} ;
    \node (v2e5) [circle,draw=black,right of=se5,node distance = 0.7cm,label=$v_2$] {};
  \node (v1e5) [circle,draw=black,above of=v2e5,node distance = 0.7cm,label=$v_1$]{};
  \node (v3e5) [circle,draw=black,fill=blue!50,below of=v2e5,node distance = 0.7cm,label=$v_3$] {};
  \node (v4e5) [circle,draw=black,right of=v1e5,node distance = 0.7cm,label=$v_4$] {};
  \node (v5e5) [circle,draw=black,right of=v3e5,node distance = 0.7cm,fill=black!50,label=$v_5$] {};
  \node (v6e5) [circle,draw=black,right of=v4e5,node distance = 0.7cm,label=$v_6$] {};
  \node (v7e5) [circle,draw=black,right of=v2e5,node distance = 1.4cm,label=$v_7$] {};

  \draw [-latex](se5)--(v1e5);
  \draw [-latex](se5)--(v2e5);
  \draw [-latex](se5)--(v3e5);
  \draw [-latex](v1e5)--(v4e5);
  \draw [-latex](v4e5)--(v6e5);
  \draw [-latex](v1e5)--(v7e5);
  \draw [-latex](v2e5)--(v7e5);
  \draw [-latex,thick,color=blue!50](v3e5)--(v5e5);
  \draw [-latex](v5e5)--(v7e5);

  \node (se6)  [circle,draw=black,label=$s$] at (-14,-12) {} ;
    \node (v2e6) [circle,draw=black,right of=se6,node distance = 0.7cm,label=$v_2$] {};
  \node (v1e6) [circle,draw=black,above of=v2e6,node distance = 0.7cm,label=$v_1$]{};
  \node (v3e6) [circle,draw=black,below of=v2e6,fill=red!50,node distance = 0.7cm,label=$v_3$] {};
  \node (v4e6) [circle,draw=black,right of=v1e6,node distance = 0.7cm,label=$v_4$] {};
  \node (v5e6) [circle,draw=black,right of=v3e6,node distance = 0.7cm,fill=black!50,label=$v_5$] {};
  \node (v6e6) [circle,draw=black,right of=v4e6,node distance = 0.7cm,label=$v_6$] {};
  \node (v7e6) [circle,draw=black,right of=v2e6,node distance = 1.4cm,label=$v_7$] {};

  \draw [-latex](se6)--(v1e6);
  \draw [-latex](se6)--(v2e6);
  \draw [-latex](se6)--(v3e6);
  \draw [-latex](v1e6)--(v4e6);
  \draw [-latex](v4e6)--(v6e6);
  \draw [-latex](v1e6)--(v7e6);
  \draw [-latex](v2e6)--(v7e6);
  \draw [-latex,thick,color=red!50](v3e6)--(v5e6);
  \draw [-latex](v5e6)--(v7e6);


  \node (ev6) at (-7,-5) {$v_6$} ;

  \node (se7)  [circle,draw=black,label=$s$] at (-9,-8) {} ;
    \node (v2e7) [circle,draw=black,right of=se7,node distance = 0.7cm,label=$v_2$] {};
  \node (v1e7) [circle,draw=black,above of=v2e7,node distance = 0.7cm,label=$v_1$]{};
  \node (v3e7) [circle,draw=black,below of=v2e7,node distance = 0.7cm,label=$v_3$] {};
  \node (v4e7) [circle,draw=black,right of=v1e7,node distance = 0.7cm,label=$v_4$] {};
  \node (v5e7) [circle,draw=black,right of=v3e7,node distance = 0.7cm,label=$v_5$] {};
  \node (v6e7) [circle,draw=black,right of=v4e7,node distance = 0.7cm,fill=black!50,label=$v_6$] {};
  \node (v7e7) [circle,draw=black,right of=v2e7,node distance = 1.4cm,label=$v_7$] {};

  \draw [-latex](se7)--(v1e7);
  \draw [-latex](se7)--(v2e7);
  \draw [-latex](se7)--(v3e7);
  \draw [-latex](v1e7)--(v4e7);
  \draw [-latex](v4e7)--(v6e7);
  \draw [-latex](v1e7)--(v7e7);
  \draw [-latex](v2e7)--(v7e7);
  \draw [-latex](v3e7)--(v5e7);
  \draw [-latex](v5e7)--(v7e7);

  \node (se8)  [circle,draw=black,label=$s$] at (-9,-12) {} ;
    \node (v2e8) [circle,draw=black,right of=se8,node distance = 0.7cm,label=$v_2$] {};
  \node (v1e8) [circle,draw=black,above of=v2e8,node distance = 0.7cm,label=$v_1$]{};
  \node (v3e8) [circle,draw=black,below of=v2e8,node distance = 0.7cm,label=$v_3$] {};
  \node (v4e8) [circle,draw=black,right of=v1e8,node distance = 0.7cm,label=$v_4$] {};
  \node (v5e8) [circle,draw=black,right of=v3e8,node distance = 0.7cm,label=$v_5$] {};
  \node (v6e8) [circle,draw=black,right of=v4e8,node distance = 0.7cm,fill=black!50,label=$v_6$] {};
  \node (v7e8) [circle,draw=black,right of=v2e8,node distance = 1.4cm,label=$v_7$] {};

  \draw [-latex](se8)--(v1e8);
  \draw [-latex](se8)--(v2e8);
  \draw [-latex](se8)--(v3e8);
  \draw [-latex](v1e8)--(v4e8);
  \draw [-latex](v4e8)--(v6e8);
  \draw [-latex](v1e8)--(v7e8);
  \draw [-latex](v2e8)--(v7e8);
  \draw [-latex](v3e8)--(v5e8);
  \draw [-latex](v5e8)--(v7e8);


  \node (ev7) at (-2,-5) {$v_7$} ;

  \node (se9)  [circle,draw=black,label=$s$] at (-4,-8) {} ;
    \node (v2e9) [circle,draw=black,right of=se9,fill=blue!50,node distance = 0.7cm,label=$v_2$] {};
  \node (v1e9) [circle,draw=black,above of=v2e9,fill=blue!50,node distance = 0.7cm,label=$v_1$]{};
  \node (v3e9) [circle,draw=black,below of=v2e9,node distance = 0.7cm,label=$v_3$] {};
  \node (v4e9) [circle,draw=black,right of=v1e9,node distance = 0.7cm,label=$v_4$] {};
  \node (v5e9) [circle,draw=black,right of=v3e9,node distance = 0.7cm,label=$v_5$] {};
  \node (v6e9) [circle,draw=black,right of=v4e9,node distance = 0.7cm,label=$v_6$] {};
  \node (v7e9) [circle,draw=black,right of=v2e9,node distance = 1.4cm,fill=black!50,label=$v_7$] {};

  \draw [-latex](se9)--(v1e9);
  \draw [-latex](se9)--(v2e9);
  \draw [-latex](se9)--(v3e9);
  \draw [-latex](v1e9)--(v4e9);
  \draw [-latex](v4e9)--(v6e9);
  \draw [-latex,thick,color=blue!50](v1e9)--(v7e9);
  \draw [-latex,thick,color=blue!50](v2e9)--(v7e9);
  \draw [-latex](v3e9)--(v5e9);
  \draw [-latex](v5e9)--(v7e9);

  \node (se0)  [circle,draw=black,label=$s$] at (-4,-12) {} ;
    \node (v2e0) [circle,draw=black,fill=red!50,right of=se0,node distance = 0.7cm,label=$v_2$] {};
  \node (v1e0) [circle,draw=black,fill=red!50,above of=v2e0,node distance = 0.7cm,label=$v_1$]{};
  \node (v3e0) [circle,draw=black,below of=v2e0,node distance = 0.7cm,label=$v_3$] {};
  \node (v4e0) [circle,draw=black,right of=v1e0,node distance = 0.7cm,label=$v_4$] {};
  \node (v5e0) [circle,draw=black,fill=red!50,right of=v3e0,node distance = 0.7cm,label=$v_5$] {};
  \node (v6e0) [circle,draw=black,right of=v4e0,node distance = 0.7cm,label=$v_6$] {};
  \node (v7e0) [circle,draw=black,right of=v2e0,node distance = 1.4cm,fill=black!50,label=$v_7$] {};

  \draw [-latex](se0)--(v1e0);
  \draw [-latex](se0)--(v2e0);
  \draw [-latex](se0)--(v3e0);
  \draw [-latex](v1e0)--(v4e0);
  \draw [-latex](v4e0)--(v6e0);
  \draw [-latex,thick,color=red!50](v1e0)--(v7e0);
  \draw [-latex,thick,color=red!50](v2e0)--(v7e0);
  \draw [-latex](v3e0)--(v5e0);
  \draw [-latex,thick,color=red!50](v5e0)--(v7e0);


    \end{tikzpicture}

    \caption{An example of diffusion process starting from $S = \set{s}$ under LT. The upper part describes an influence diffusion process where yellow nodes represent influenced nodes by the current time. The lower part describes what $E_{\tau_1}, E_{\tau_2-1}$ are where we use blue (red) color to represent the edges and the associated active in-neighbors in $E_{\tau_1}$ ($E_{\tau_2-1}$, respectively) for the objective black node. }
    \label{fig:illustration}

\end{figure}

Figure \ref{fig:illustration} gives an example of diffusion process and the definitions of edge-sets $E_{\tau_1}$ and $E_{\tau_2-1}$. The diffusion process is illustrated by the upper four figures, where the set $S_\tau$ of influenced nodes by time $\tau$ is yellow colored. The lower five columns represent the sets $E_{\tau_1}, E_{\tau_2-1}$ for different nodes. For example, node $v_7$ has active in-neighbors starting from $\tau=1$, thus $\tau_1(v_7) = 1$ and $E_{\tau_1(v_7)}(v_7) = \set{e_{u,v_7} : u \in N(v_7)\cap S_{1}} = \set{e_{v_1, v_7},e_{v_2, v_7}}$. And $v_7$ is influenced at $\tau = 3$ thus $\tau_2(v_7) = 3$ and $E_{\tau_2(v_7)-1}(v_7) = \set{e_{u,v_7} : u \in N(v_7)\cap S_{2}} = \set{e_{v_1, v_7}, e_{v_2, v_7}, e_{v_5, v_7}}$. Node $v_6$ has no active in-neighbors, thus $\tau_1(v_6) = \tau_2(v_6) = D+1$, both its $E_{\tau_1(v_6)}(v_6)$ and $E_{\tau_2(v_6)-1}(v_6)$ are empty sets.

The above describes how to distill key observations for a diffusion under the LT model and also explains the update rule in the design of the algorithm. Denote $\tau_1,\tau_2,E_{\tau}$ at round $t$ as $\tau_{t,1},\tau_{t,2},E_{t,\tau}$ and the diffusion process at round $t$ as $S_{t,0},\ldots, S_{t,\tau}, \ldots$. Here we abuse a bit the notation $S$ to represent both the seed set and the spread set in a round when the context is clear.

Our algorithm $\ltlinucb$ is given in Algorithm \ref{alg:LinUCB}. It maintains the Gramian matrix $M_v$ and the moment vector $b_v$ of regressand by regressors to store the information for $w_v$. At each round $t$, the learner first computes the confidence ellipsoid for $w_v$ based on the current information (line \ref{alg:linucb:compute ellipsoid}) (see the following lemma). 

\begin{lemma} 
\label{lem:ltlinucb:confidence ellipsoid}
Given $\set{(A_{t}, y_{t})}_{t=1}^{\infty}$ with $A_{t} \in \set{0,1}^{N}$ and $y_{t}\in \{0,1\}$ as a Bernoulli random variable with $\EE{y_{t} \mid A_{1}, y_{1}, \ldots, A_{t-1}, y_{t-1}, A_{t}} = A_{t}^\top w_v$, let $M_{t} = I + \sum_{s=1}^t A_{s} A_{s}^\top$ and $\hat{w}_{t} = M_{t}^{-1}\bracket{\sum_{s=1}^t A_{s} y_{s}}$ be the linear regression estimator. Then with probability at least $1-\delta$, for all $t \ge 1$, it holds that $w_v$ lies in the confidence set
\begin{align*}
  \tilde{\cC}_{t} := \set{w' \in [0,1]^{N} : \norm{w' - \hat{w}_{t}}_{M_{t}} \leq \sqrt{N \log(1+tN)+2\log\frac{1}{\delta}} + \sqrt{N}}\,.
\end{align*}
\end{lemma}

\begin{algorithm}[thb!]
\caption{$\ltlinucb$}
\label{alg:LinUCB}
\begin{algorithmic}[1]
\STATE \textbf{Input:} Graph $G=(V,E)$; seed set cardinality $K$; exploration parameter $\rho_{t,v}>0$ for any $t,v$; offline oracle $\pairoracle$
\STATE \textbf{Initialize}: $M_{0,v} \gets I \in \RR^{\abs{N(v)}\times \abs{N(v)}}, b_{0,v} \gets 0 \in \RR^{\abs{N(v)} \times 1}, \hat{w}_{0,v} \gets 0 \in \RR^{\abs{N(v)} \times 1}$ for any node $v \in V$ \\
\FOR {$t = 1,2,3, \ldots$}
\STATE Compute the confidence ellipsoid $\cC_{t,v} = \set{w_v' \in [0,1]^{\abs{N(v)} \times 1} : \norm{w_v' - \hat{w}_{t-1,v}}_{M_{t-1,v}} \leq \rho_{t,v}}$ for any node $v \in V$
\label{alg:linucb:compute ellipsoid}
\STATE Compute the pair $(S_t,w_t)$ by $\pairoracle$ with confidence set $\cC_t=\set{\cC_{t,v}}_{v \in V}$ and seed set cardinality $K$
\label{alg:linucb:compute seedset}
\STATE Select the seed set $S_t$ and observe the feedback
\label{alg:linucb:choose seedset}
\STATE // Update  
\label{alg:linucb:update start} 
\FOR {node $v \in V$ }
  \STATE Initialize $A_{t,v} \gets 0 \in \RR^{\abs{N(v)} \times 1}$, $y_{t,v} \gets 0 \in \RR$
  \STATE Uniformly randomly choose $\tau \in \set{ \tau': \tau_{t,1}(v) \le \tau' \le \tau_{t,2}(v)-1 } $ \label{alg:linucb:choosetau}
  \IF {$v$ is influenced and $\tau= \tau_{t,2}(v)-1$} \label{alg:linkbeginupdateAu}
      \STATE $A_{t,v} = \chi\bracket{E_{t,\tau}(v)}$,  $~~y_{t,v} = 1$ \label{alg:linucb:update:succ}
  \ELSIF {$\tau = \tau_1(v) ,\ldots,\tau_2(v)-2 $ or $\tau=\tau_2(v)-1$ but $v$ is not influenced}
      \STATE $A_{t,v} = \chi\bracket{E_{t,\tau}(v)}$,  $~~y_{t,v} = 0$ \label{alg:linucb:update:fail}
  \ENDIF \label{alg:linefinishupdateAu}
  \STATE $M_{t,v} \gets M_{t-1,v} + A_{t,v}A_{t,v}^{\top}, ~~ b_{t,v} \gets b_{t-1,v} + y_{t,v} A_{t,v} , ~~\hat{w}_{t,v} = M_{t,v}^{-1}b_{t,v}$
\ENDFOR
\label{alg:linucb:update end}
\ENDFOR
\end{algorithmic}
\end{algorithm}

This lemma is a direct corollary of \cite[Theorem 2]{Linearbandits} for the concentration property of the weight vector $w_v$. Thus when $\rho_{t,v} \ge \sqrt{\abs{N(v)} \log(1+t|N(v)|)+2\log\frac{1}{\delta}} + \sqrt{\abs{N(v)}}$, the true weight vector $w_v$ lies in the confidence set $\cC_{t,v}$ (line \ref{alg:linucb:compute ellipsoid}) for any $t$ with probability at least $1-\delta$.

Given the confidence set $\cC_v$ for $w_v$, the algorithm expects to select the seed set by solving the \textit{weight-constrained influence maximization} (WCIM) problem
\begin{align}
	\argmax_{(S,w'): S \in \cA, w' \in \cC} \ r(S,w') \label{eq:WCIM}\,.
\end{align}
This (offline) optimization problem turns out to be highly nontrivial. 
Since we want to focus more on the online learning solution, we defer the full discussion on the offline optimization, including its general difficulty and our proposed approximate algorithms for certain graph classes such as directed acyclic graphs to 
\ifsup
Appendix \ref{sec:app_pairoracle}.
\else
supplementary materials.
\fi

Suppose its best solution is $(S^{\popt}_{\cC},w^{\popt}_{\cC})$ where `P' stands for `pair'. Let $\pairoracle$ be an offline oracle to solve the optimization problem. We say $\pairoracle$ is an $(\alpha, \beta)$-approximation oracle if $\PP{r(S',w') \ge \alpha \cdot r(S^{\popt}_{\cC},w^{\popt}_{\cC})} \ge \beta$ where $(S',w')$ is an output by the oracle when the confidence set is ${\cC}$. Then the algorithm runs with the seed set output by the $\pairoracle$ and the confidence set $\cC_t=\set{\cC_{t,v}}_{v \in V}$ (line \ref{alg:linucb:compute seedset}). 

After observing the diffusion process (line \ref{alg:linucb:choose seedset}), For each node $v$ who has active in-neighbors, we randomly choose its active in-neighbors at time step $\tau_1(v),\ldots,\tau_2(v)-1$ to update (line \ref{alg:linucb:choosetau}). Specifically, if $v$ is influenced and $\tau=\tau_2(v)-1$, then it means that the set of active in-neighbors at time step $\tau$ succeeds to activate $v$, thus we use $(\chi(E_{t,\tau}(v)), 1)$ to update (line \ref{alg:linucb:update:succ}); if $\tau = \tau_1(v), \ldots, \tau_2(v) - 2$ or $\tau=\tau_2(v)-1$ but node $v$ is not influenced, it means that the set of active in-neighbors at $\tau$ fail to activate node $v$, thus we use $(\chi(E_{t,\tau}(v)), 0)$ to update (line \ref{alg:linucb:update:fail}). These updates are consistent with the distilled observations we get for nodes who have active in-neighbors.
For node $v$ who has no active in-neighbors, we have no obervation on $w_v$ and not update on it since the set $\set{\tau': \tau_1(v) \le \tau' \le \tau_2(v)-1}$ is an empty set in this case. 


For example in Figure \ref{fig:illustration}, node $v_7$ has active in-neighbors from $\tau_1(v_7)=1$ and is influenced at $\tau_2(v_7)=3$. The $\ltlinucb$ will uniformly randomly choose $\tau\in\set{1,2}$ (line \ref{alg:linucb:choosetau}). It updates $(A_{v_7} = \chi(E_{1}(v_7)), y_{v_7} = 0)$ if $\tau=1$ (line \ref{alg:linucb:update:fail}) and $(A_{v_7} = \chi(E_{2}(v_7)), y_{v_7} = 1)$ otherwise (line \ref{alg:linucb:update:succ}). For nodes $v_1,v_2,v_3$, they all have $\tau_1 = 0$ and $\tau_2 = 1$. Thus for these three nodes, the algorithm chooses $\tau=0$ (line \ref{alg:linucb:choosetau}) and updates $(A_{v} = \chi(E_{0}(v)), y_{v} = 1)$ (line \ref{alg:linucb:update:succ}). 
Node $v_4$ has active in-neighbors from $\tau_1(v_4) = 1$ but is not influenced finally, the algorithm will randomly choose $\tau\in \set{1,2\ldots,D}$ and update $(A_{v_4} = \chi(E_{\tau}(v_4)), y_{v_4} = 0)$ (line \ref{alg:linucb:update:fail}). Node $v_6$ has no active in-neighbors, so we have no observation for its weight vector and will not update on it.


\subsection{Regret Analysis}

We now provide the group observation modulated (GOM) bounded smoothness property for LT model, an important relationship of the influence spreads under two weight vectors. 
It plays a crucial role in the regret analysis and states that the difference of the influence spread $r(S, w)$ under two weight vectors can be bounded in terms of the weight differences of the distilled observed edge sets under one weight vector. 
It is conceptually similar to the triggering probability modulate (TPM) bounded smoothness condition under the IC model with edge-level feedback \cite{WeiChen2017}, but its derivation and usage are quite different.
For the seed set $S$, define the set of all nodes related to a node $v$, $V_{S,v}$, to be the set of nodes that are on any path from $S$ to $v$ in graph $G$.

\begin{theorem}\label{theorem:TPM}
(GOM bounded smoothness) For any two weight vectors $w,w' \in [0,1]^m$ with $\sum_{u\in N(v)}w(e_{u,v}) \le 1$, the difference of their influence spread for any seed set $S$ can be bounded as
\begin{align}
&\abs{r(S, w')-r(S,w)} \le  \mathbb{E}\Bigg[ \sum_{v \in V \setminus S}  \sum\limits_{u \in V_{S,v}} \sum_{\tau = \tau_1(u)}^{\tau_2(u)-1}  \abs{ \sum_{e \in E_{\tau}(u)} (w'(e)-w(e))}  \Bigg] ,
\end{align} 
where the definitions of $\tau_1(u), \tau_2(u)$ and $E_\tau(u)$ are all under weight vector $w$, and the expectation is taken over the randomness of the thresholds on nodes.
\end{theorem}


This theorem connects the reward difference with weight differences on the distilled observations,
which are also the information used to update the algorithm (line \ref{alg:linucb:update start}-\ref{alg:linucb:update end}). It links the effective observations, updates of the algorithm and the regret analysis.
The proof needs to deal with intricate dependency among activation events, and is put in
\ifsup  
Appendix \ref{app:gom proof}.
\else 
the supplementary materials. 
\fi
due to the space constraint.

For seed set $S \in \cA$ and node $u \in V \setminus S$, define 
$
N_{S,u} := \sum_{v\in V\backslash S} \bOne{u\in V_{S, v}} \le n-K
$
to be the number of nodes that $u$ is relevant to. Then for the vector $N_{S} = (N_{S,u})_{u \in V}$, define the upper bound of its $L^2$-norm over all feasible seed sets
\begin{align*}
\gamma(G) := \max_{S \in \cA} \sqrt{\sum_{u\in V} N_{S,u}^2} \le  (n-K)\sqrt{n} = O(n^{3/2})\,,
\end{align*}
which is a constant related to the graph. Then we have the following regret bound.

\begin{theorem}\label{main theorem}
Suppose the $\ltlinucb$ runs with an $(\alpha,\beta)$-approximation $\pairoracle$ and parameter
$
\rho_{t,v} = \rho_t =   \sqrt{n \log(1+tn)+2\log\frac{1}{\delta}} + \sqrt{n}
$ for any node $v \in V$.
Then the $\alpha \beta$-scaled regret satisfies
\begin{align}
  R(T) &\le 2\rho_T \gamma(G) D \sqrt{mnT \log(1+T) / \log(1+n)} + n\delta \cdot T(n-k)\,. \label{eq: regret bound}
\end{align}

When $\delta = 1/ (n\sqrt{T})$, $R(T) \le C \cdot \gamma(G) \  D n\sqrt{mT} \log(T)$ for some universal constant $C$.
\end{theorem}

Due to space limits, the proof and the detailed discussions, as well as the values of $\gamma(G)$, are put in 
\ifsup
Appendix \ref{sec:app linucb}.
\else
supplementary materials.
\fi


\section{The Explore-then-Commit Algorithm}\label{sec:IMETC}

This section presents the explore-then-commit (ETC) algorithm for OIM. Though simple, it is efficient and model independent, applying to both LT and IC model with less requirement on feedback and offline computation.

Recall that under LT model, a node $v$ is activated if the sum of weights from active in-neighbors exceeds the threshold $\theta_v$, which is uniformly drawn from $[0,1]$. Since the feedback is node-level, if the activated node $v$ has more than one active in-neighbors, then we can only observe the group influence effect of her active in-neighbors instead of each single in-neighbor. A simple way to overcome this limitation and manage to observe directly the single weight $w(e_{u,v})$ is to select a single seed $\set{u}$ and take only the first step influence as feedback, which formulates our $\oimetc$ algorithm (Algorithm \ref{alg:IMETC}), representing the ETC algorithm of the OIM problem.

Our $\oimetc$ takes the exploration budget $k$ as input parameter such that each node $u$ is selected as the (single) seed for $k$ rounds (line \ref{algo:etc:choose node}). 
For each round in which $u$ is the seed, each outgoing neighbor (shortened as out-neighbor) $v \in N^{\mathrm{out}}(u)$ will be activated in the first step with probability $\PP{w(e_{u,v}) > \theta_v} = w(e_{u,v})$ since the threshold $\theta_v$ is independently uniformly drawn from $[0,1]$. Thus the first-step node-level feedback is actually edge-level feedback and we can observe the independent edges from the first-step feedback (line \ref{algo:etc:get observation}). Since each node is selected $k$ times, we have $k$ observations of Bernoulli random variables with expectation $w(e_{u,v})$ in this \textit{exploration} phase. Then we take the empirical estimate $\hat{w}(e)$ for each $w(e)$ (line \ref{algo:etc:compute hat w}) after the exploration and run with the seed set output by the offline $\oracle$ (line \ref{algo:etc:compute hat S}) for the remaining $T-nk$ \textit{exploitation} rounds (line \ref{algo:etc:exploit}).
We assume the offline $\oracle$ is $(\alpha,\beta)$-approximation. 

Since it only needs the first step of the diffusion process and calls only once of the usual IM oracle, it is efficient and has less requirement.
By selecting reasonable $k$, we can derive good regret bounds. Before that we need two definitions.

\begin{algorithm}[t]{}
\caption{$\oimetc$}
\label{alg:IMETC}
\begin{algorithmic}[1]
\STATE \textbf{Input:} $G=(V,E)$, seed size $K$, exploration budget $k$, time horizon $T$, offline oracle $\oracle$
\FOR{$s \in [k], u \in V$}
  \STATE Choose $\set{u}$ as the seed set
  \label{algo:etc:choose node}
  \STATE $X_s(e_{u,v}) := \bOne{v \text{ is activated}}$ for any $v  \in N^{\mathrm{out}}(u)$
  \label{algo:etc:get observation}
\ENDFOR
\STATE Compute $\hat{w}(e) := \frac{1}{k}\sum_{s=1}^{k}X_s(e)$ for any $e \in E$
\label{algo:etc:compute hat w}
\STATE $\hat{S} = \oracle(\hat{w})$
\label{algo:etc:compute hat S}
\FOR{the remaining $T-nk$ rounds}
  \STATE Choose $\hat{S}$ as the seed set
  \label{algo:etc:exploit}
\ENDFOR
\end{algorithmic}
\end{algorithm}

\begin{definition}\label{def:BadSeedSet}
(Bad seed set) \  A seed set $S$ is bad if $ r(S, w) < \alpha \cdot \opt_w$. The set of bad seed sets is $\cS_B := \set{S \mid r(S, w) < \alpha \cdot \opt_w }$. 
\end{definition}

\begin{definition}\label{def:Gaps}
(Gaps of bad seed sets)\ For a bad seed set $S \in \mathcal{S}_B$, its gap is defined as $\Delta_S := \alpha \cdot \opt_w - r(S, w)$. The maximum and minimum gap are defined as
\begin{align}
\label{eq:Delta_max} \Delta_{\max} := \alpha \cdot \opt_w - \min \set{r(S, w) \mid S \in \mathcal{S}_B }  \,,\\
\label{eq:Delta_min} \Delta_{\min} := \alpha \cdot \opt_w - \max \set{r(S, w) \mid S \in \mathcal{S}_B }  \,.
\end{align}
\end{definition}

\begin{theorem}
\label{thm:etc}
When $k = \max\set{1, \frac{2m^2 n^2}{\Delta_{\min}^2}\ln\bracket{\frac{T \Delta_{\min}^2}{m n^3}}  }$, the $\alpha\beta$-scaled regret bound of our $\oimetc$ algorithm over $T$ rounds satisfies
\begin{align}\label{eq:ETC_dependent}
R(T) &\leq \min\set{T\Delta_{\max}, n\Delta_{\max} + \frac{2m^2 n^3\Delta_{\max}}{\Delta_{\min}^2}\bracket{1 + \max\set{0,\ln\bracket{\frac{T\Delta_{\min}^2}{mn^3}} }   }} \notag\\
&= O\bracket{\frac{m^2 n^3\Delta_{\max}}{\Delta_{\min}^2} \ln(T)}\,.
\end{align}
When $k = 3.9 (m^2T/n)^{2/3}$, the $\alpha\beta$-scaled regret bound of $\oimetc$ algorithm over $T$ rounds satisfies
\begin{align}\label{eq:ETC_independent}
R(T) &\le 3.9(mn)^{4/3} T^{2/3} + 1 =O\bracket{(mn)^{4/3} T^{2/3}} \,.
\end{align}
\end{theorem}

The proof of the problem-dependent bound follows routine ideas of ETC algorithms but the proof of the problem-independent bound is new. The proofs and discussions are put in 
\ifsup 
Appendix \ref{sec:app etc}. 
\else supplementary materials. 
\fi

\section{Conclusion}
\label{sec:conclusion}
In this paper, we formulate the problem of OIM under LT model with node-level feedback and design how to distill effective information from observations. 
We prove a novel GOM bounded smoothness property for the spread function, which relates the limited observations, algorithm updates and the regret analysis. We propose $\ltlinucb$ algorithm, provide rigorous theoretical analysis and show a competitive regret bound of $O(\mathrm{poly}(m)\sqrt{T} \ln(T))$. 
Our $\ltlinucb$ is the first algorithm for LT model with such regret order. Besides, we design $\oimetc$ algorithm with theoretical analysis on its distribution-dependent and distribution-independent regret bounds. 
The algorithm is efficient, applies to both LT and IC models, and has less requirements on feedback and offline computation.

In studying the OIM with LT model, we encounter an optimization problem of weight-constrained influence maximization (WCIM). Reconsidering an (offline) optimization problem by relaxing some fixed parameter to elements of a convex set is expected to be common in online learning. So we believe this problem could have independent interest. Also the OIM problem under IC model with node-level feedback is an interesting future work. 
Our regret analysis goes through thanks to the linearity of the LT model. 
But the local triggering is nonlinear for IC model, and thus we expect more challenges in the design and analysis of IC model with node-level feedback.
Applying Thompson sampling to influence maximization is also an interesting future direction, but it could also be challenging, since it may not work well with offline approximation oracles as pointed out in~\cite{WangC18}.

\newpage
 \section*{Acknowledgement}
 We thank Chihao Zhang for valuable discussions.

\bibliographystyle{plain}
\bibliography{ref}

\ifsup
\appendix

\section{Analysis and Discussions of $\ltlinucb$}\label{sec:app linucb}


\subsection{Proof of Theorem \ref{theorem:TPM}}
\label{app:gom proof}

Let $r_S^v(w)$ be the probability that node $v$ will be influenced under the weight vector $w$ when the seed set is $S$. Then
\begin{align*}
&\abs{r(S, w') - r(S, w)} \\
&\qquad \le \sum_{v \in V \setminus S} \abs{r_S^v(w') - r_{S}^v(w)} \\
&\qquad = \sum_{v \in V \setminus S}\mathbb{E}_{\theta \sim (\cU[0,1])^n} \left[\bOne{ v\text{ is influenced under } {w', \theta} }\neq \bOne{ v\text{ is influenced under } w, \theta}\right]\,,
\end{align*}
where we use $\cU[0,1]$ to denote the uniform distribution on the interval $[0,1]$. The reason that the activation of $v$ is different under $w$ and $w'$ must be that during the propagation from $S$ to $v$, at some step $\tau$ and some node $u \in V_{S,v}$, the activation of $u$ is different.
We enumerate $u \in V_{S,v}$ and enumerate $\tau$ from $1$ to $D$ to bound the above probability. Recall that $D$ is the propagation diameter.
Henceforth in this section, parameters $w$, $w'$, $S$, and $v$ are all fixed.
All the randomness comes from $\theta \sim (\cU[0,1])^n$, and once $\theta$ is determined, the diffusion process is determined.
Thus, we could assume that every event is a subset of $[0,1]^n$.
Define the following event, given the seed set $S$ and target node $v$:
\begin{equation*}
\cE_0 = \{\theta \mid \bOne{ v\text{ is influenced under } {w', \theta} }\neq \bOne{ v\text{ is influenced under } w, \theta} \}.
\end{equation*}
Thus
\begin{align}
\abs{r(S, w') - r(S, w)} & \le \sum_{v \in V \setminus S} \Pr_{\theta \sim (\cU[0,1])^n}\{\cE_0\}. \label{eq:reward2individualtarget}
\end{align}
Let $\Phi(w,\theta) = (S_0=S, S_1, \ldots, S_D)$ be the sequence of activation sets given weight factor $w$ and threshold factor $\theta$.
Let $\Phi_i(w,\theta) = S_i$ be the set of nodes activated by time step $i$.
For every node $u\in V_{S,v}$, we define the event that $u$ is the first node that has different activation under $w$ and $w'$.
\begin{align*}
\cE_1(u) = \{ \theta \mid \exists \tau\in [D], \forall \tau' < \tau,\ &\Phi_{\tau'}(w,\theta) = \Phi_{\tau'}(w',\theta), \\
&u \in (\Phi_{\tau}(w,\theta)\setminus \Phi_{\tau}(w',\theta)) \cup (\Phi_{\tau}(w',\theta) \setminus \Phi_{\tau}(w,\theta)) \} \,.
\end{align*}
It is clear that 
\begin{equation}
\cE_0 \subseteq \bigcup_{u\in V_{S,v}} \cE_1(u)\,. \label{eq:eventu}
\end{equation}
Note that for each node $u\in V_{S,v}$, $u$ may be activated at different time steps from different paths, or not activated at all. 
Thus, the fact that $u$ is not activated at one time step may have implications on $u$'s activations at other time steps, and thus we need to carefully 
	classify the activation of $u$ in order to bound the probability of $\cE_1(u)$.
Define the following events for each $\tau \in [D]$:
\begin{align*}
\cE_{2,0}(u,\tau) & = \{ \theta \mid \forall \tau' < \tau, \Phi_{\tau'}(w,\theta) = \Phi_{\tau'}(w',\theta), u \not\in \Phi_{\tau-1}(w,\theta) \}\,, \\
\cE_{2,1}(u,\tau) & = \{\theta \mid \forall \tau' < \tau, \Phi_{\tau'}(w,\theta) = \Phi_{\tau'}(w',\theta), u \in \Phi_{\tau}(w,\theta)\setminus \Phi_{\tau}(w',\theta) \}\,,\\
\cE_{2,2}(u,\tau) & = \{\theta \mid  \forall \tau' < \tau, \Phi_{\tau'}(w,\theta) = \Phi_{\tau'}(w',\theta), u \in \Phi_{\tau}(w',\theta) \setminus \Phi_{\tau}(w,\theta)  \}\,, \\
\cE_{3,1}(u,\tau) & = \{\theta \mid u \in \Phi_{\tau}(w,\theta)\setminus \Phi_{\tau}(w',\theta) \}\,, \\
\cE_{3,2}(u,\tau) & = \{\theta \mid u \in \Phi_{\tau}(w',\theta)\setminus \Phi_{\tau}(w,\theta) \}\,. 
\end{align*}
Note that all the events $\cE_{2,1}(u,\tau),\cE_{2,2}(u,\tau)$ for $\tau \in [D]$ are mutually exclusive.
Therefore,
\begin{equation} \label{eq:sumevents}
\Pr_{\theta \sim (\cU[0,1])^n}\{\cE_1(u)\} = \sum_{\tau=1}^D \Pr_{\theta \sim (\cU[0,1])^n}\{\cE_{2,1}(u,\tau) \} + 
	\sum_{\tau=1}^D \Pr_{\theta \sim (\cU[0,1])^n}\{\cE_{2,2}(u,\tau) \}\,.
\end{equation}

We first bound $\Pr_{\theta \sim (\cU[0,1])^n}\{\cE_{2,1}(u,\tau) \}$. Now fix all entries of $\theta$ vector except $\theta_u$, denoted as $\theta_{-u}$, and the corresponding subevent of
	$\cE_{2,1}(u,\tau)$ is defined as $\cE_{2,1}(u,\tau, \theta_{-u}) \subseteq \cE_{2,1}(u,\tau)$.
Similarly $\cE_{2,0}(u,\tau, \theta_{-u}) \subseteq \cE_{2,0}(u,\tau)$ and $\cE_{3,1}(u,\tau, \theta_{-u}) \subseteq \cE_{3,1}(u,\tau)$ are defined. Also $\cE_{2,1}(u,\tau, \theta_{-u}) = \cE_{2,0}(u,\tau, \theta_{-u}) \cap \cE_{3,1}(u,\tau, \theta_{-u})$.


Note that $\cE_{2,1}(u,\tau) = \cE_{2,0}(u,\tau) \cap \cE_{3,1}(u,\tau)$, and $\cE_{2,2}(u,\tau) = \cE_{2,0}(u,\tau) \cap \cE_{3,2}(u,\tau)$.
Thus
\begin{align}
\Pr_{\theta \sim (\cU[0,1])^n}\{\cE_{2,1}(u,\tau) \} = \Pr_{\theta \sim (\cU[0,1])^n}\{\cE_{2,0}(u,\tau) \} \cdot 
		\Pr_{\theta \sim (\cU[0,1])^n} \{ \cE_{3,1}(u,\tau) \mid \cE_{2,0}(u,\tau)\} \,.
\end{align}


Then 
\begin{equation} \label{eq:E21-u}
\Pr_{\theta_u \sim \cU[0,1]}\{\cE_{2,1}(u,\tau, \theta_{-u}) \} = \Pr_{\theta_u \sim \cU[0,1]}\{\cE_{2,0}(u,\tau, \theta_{-u}) \} \cdot 
\Pr_{\theta_u \sim \cU[0,1]} \{ \cE_{3,1}(u,\tau, \theta_{-u}) \mid \cE_{2,0}(u,\tau, \theta_{-u})\}.
\end{equation}
Symmetric equations also hold for $\cE_{2,2}(u,\tau, \theta_{-u})$.

In the event $\cE_{2,0}(u,\tau, \theta_{-u})$, all entries in $\theta$ vector is fixed except for $\theta_u$.
It is easy to check that if $(\theta_{-u}, \theta_u) \in \cE_{2,0}(u,\tau, \theta_{-u})$, then for all $\theta'_u \ge \theta_u$, 
	$(\theta_{-u}, \theta'_u) \in \cE_{2,0}(u,\tau, \theta_{-u})$.
This is because the $\theta_{-u}$ is fixed, so the activations of all nodes other than $u$ have the same conditions, while for $u$
	it is even harder to activate $u$ with larger $\theta_u$.
Therefore, in $\cE_{2,0}(u,\tau, \theta_{-u})$, the entry on $\theta_u$ must be an interval from some lowest value to $1$.
Let $\theta_{u,2,0}(\tau, \theta_{-u})$ be the left point of this interval.
That is $\cE_{2,0}(u,\tau, \theta_{-u}) = \{(\theta_{-u}, \theta_u) \mid \theta_u > \theta_{u,2,0}(\tau, \theta_{-u})\}$.
Then we have 
\begin{align}
	\Pr_{\theta_u \sim \cU[0,1]}\{\cE_{2,0}(u,\tau, \theta_{-u}) \} = 1 - \theta_{u,2,0}(\tau, \theta_{-u}) \label{eq:E20-theta-u}.
\end{align}
For now, let's first assume that $\cE_{2,0}(u,\tau, \theta_{-u}) \ne \emptyset$, that is, $\theta_{u,2,0}(\tau, \theta_{-u}) < 1$.
In the event $\cE_{2,0}(u,\tau, \theta_{-u})$, we know that the set of activated nodes until $\tau-1$ are the same under both $w$ and $w'$ and $u$ is not activated by
	time $\tau-1$, and since $\theta_{-u}$ is fixed, the set of activated nodes by time $\tau-1$ are all fixed.
We denote the set of nodes activated by time step $i$ under event $\cE_{2,0}(u,\tau, \theta_{-u})$ as $\Phi_i(\cE_{2,0}(u,\tau, \theta_{-u}))$.

Now conditioned on the event $\cE_{2,0}(u,\tau, \theta_{-u})$, we consider event $\cE_{3,1}(u,\tau, \theta_{-u}) \cup \cE_{3,2}(u,\tau, \theta_{-u})$.
This means that conditioned on $\theta_u > \theta_{u,2,0}(\tau, \theta_{-u})$ and a fixed activated set $\Phi_{\tau-1}(\cE_{2,0}(u,\tau, \theta_{-u}))$ by time $\tau-1$,
	$u$ is activated at step $\tau$ under one of $w$ and $w'$ but not both.
According to the information diffsuion under the LT model, this means either the following inequality holds,
\begin{align*}
	\sum_{u' \in \Phi_{\tau-1}(\cE_{2,0}(u,\tau, \theta_{-u})) \cap N(u)} w'(e_{u', u}) < \theta_u \le \sum_{u' \in \Phi_{\tau-1}(\cE_{2,0}(u,\tau, \theta_{-u})) \cap N(u)} w(e_{u', u})\,.
\end{align*}
or the following holds
\begin{align*}
\sum_{u' \in \Phi_{\tau-1}(\cE_{2,0}(u,\tau, \theta_{-u})) \cap N(u)} w(e_{u', u}) < \theta_u \le \sum_{u' \in \Phi_{\tau-1}(\cE_{2,0}(u,\tau, \theta_{-u})) \cap N(u)} w'(e_{u', u})\,.
\end{align*}

This in turn implies that 
\begin{align*}
& \Pr_{\theta_u \sim \cU[0,1]} \{ \cE_{3,1}(u,\tau, \theta_{-u}) \cup \cE_{3,2}(u,\tau, \theta_{-u}) \mid \cE_{2,0}(u,\tau, \theta_{-u})\} \\
& =	\frac{\left | \sum_{u' \in \Phi_{\tau-1}(\cE_{2,0}(u,\tau, \theta_{-u})) \cap N(u)} w(e_{u', u}) -
	\sum_{u' \in \Phi_{\tau-1}(\cE_{2,0}(u,\tau, \theta_{-u})) \cap N(u)} w'(e_{u', u}) \right |}{1-\theta_{u,2,0}(\tau, \theta_{-u})}\,.
\end{align*}
Plugging the above equality and Eq.\eqref{eq:E20-theta-u} into Eq.\eqref{eq:E21-u}, and use the fact that $\cE_{3,1}(u,\tau, \theta_{-u}) $ and $\cE_{3,2}(u,\tau, \theta_{-u}) $ are mutually exclusive,  we have
\begin{align}
& \Pr_{\theta_u \sim \cU[0,1]}\{\cE_{2,1}(u,\tau, \theta_{-u}) \cup \cE_{2,2}(u,\tau, \theta_{-u}) \} \nonumber \\
& =  \left | \sum_{u' \in \Phi_{\tau-1}(\cE_{2,0}(u,\tau, \theta_{-u})) \cap N(u)} ( w(e_{u', u}) -w'(e_{u', u})) \right |\,. \label{eq:inequalityE20}
\end{align}

Note that when $\cE_{2,0}(u,\tau, \theta_{-u}) = \emptyset$, both the LHS and the RHS of the above equality is zero, so this equality holds in general.

We now need to relax event $\cE_{2,0}(u,\tau, \theta_{-u})$, since it depends on both $w$ and $w'$.
We define a new event to detach it from $w'$, 
\begin{align*}
\cE_{4,0}(u,\tau, \theta_{-u}) & = \{ \theta =(\theta_{-u}, \theta_u) \mid u \not \in \Phi_{\tau-1}(w,\theta) \}\,.
\end{align*}
It is clear that $\cE_{2,0}(u,\tau, \theta_{-u}) \subseteq \cE_{4,0}(u,\tau, \theta_{-u})$.
Moreover, when $\cE_{2,0}(u,\tau, \theta_{-u}) \ne \emptyset$, we see that both events $\cE_{2,0}(u,\tau, \theta_{-u})$ and $\cE_{4,0}(u,\tau, \theta_{-u})$
	have fixed $\theta_{-u}$ and dictate that $u$ is not activated by time $\tau-1$ under $w$.
This implies that they have the same set of nodes activated by time step $i$ for $i\le \tau-1$.
Denote $\Phi_i(\cE_{4,0}(u,\tau, \theta_{-u}))$ be this set.
The above means that for all $i\le \tau-1$, 
	$\Phi_{i}(\cE_{2,0}(u,\tau, \theta_{-u})) = \Phi_i(\cE_{4,0}(u,\tau, \theta_{-u}))$.
Therefore, we can relax Eq.\eqref{eq:inequalityE20}  to get the following.
\begin{align*} 
& \Pr_{\theta_u \sim \cU[0,1]}\{\cE_{2,1}(u,\tau, \theta_{-u}) \cup \cE_{2,2}(u,\tau, \theta_{-u})  \} \\
& \le \left | \sum_{u' \in \Phi_{\tau-1}(\cE_{4,0}(u,\tau, \theta_{-u})) \cap N(u)} ( w(e_{u', u}) -w'(e_{u', u})) \right |\,.
\end{align*}
Note that when $\cE_{2,0}(u,\tau, \theta_{-u}) = \emptyset$, the LHS of above is zero, so the inequality still holds.

Combining the above with Eq.\eqref{eq:sumevents}, we have
\begin{align*}
&\Pr_{\theta \sim (\cU[0,1])^n}\{\cE_1(u)\} \\
 = & \int_{\theta_{-u}\in [0,1]^{n-1}} \sum_{\tau=1}^D \Pr_{\theta_u \sim \cU[0,1]}\{\cE_{2,1}(u,\tau, \theta_{-u}) \cup 
	\cE_{2,2}(u,\tau, \theta_{-u})\} \, {\rm d} \theta_{-u} \\
 = & \sum_{\tau=1}^D \int_{\theta_{-u}\in [0,1]^{n-1}} \Pr_{\theta_u \sim \cU[0,1]}\{\cE_{2,1}(u,\tau, \theta_{-u}) \cup 
\cE_{2,2}(u,\tau, \theta_{-u})\} \, {\rm d} \theta_{-u} \\
 \le &\sum_{\tau=1}^D \int_{\theta_{-u}\in [0,1]^{n-1}} \left | \sum_{u' \in \Phi_{\tau-1}(\cE_{4,0}(u,\tau, \theta_{-u})) \cap N(u)} ( w(e_{u', u}) -w'(e_{u', u})) \right |
	 \, {\rm d} \theta_{-u} \\
 = &\sum_{\tau=1}^D \mathbb{E}_{\theta_{-u} \sim (\cU[0,1])^{n-1}} 
	\left[ \left | \sum_{u' \in \Phi_{\tau-1}(\cE_{4,0}(u,\tau, \theta_{-u})) \cap N(u)} ( w(e_{u', u}) -w'(e_{u', u})) \right | \right]\,.
\end{align*}
Combining the above with Eq.\eqref{eq:reward2individualtarget} and Eq.\eqref{eq:eventu}, we have
\begin{align*}
&\abs{r(S, w') - r(S, w)} \\
 \le & \sum_{v \in V \setminus S} \sum_{u\in V_{S,v}} 
	\sum_{\tau=1}^D \mathbb{E}_{\theta_{-u} \sim (\cU[0,1])^{n-1}} 
	\left[ \left | \sum_{u' \in \Phi_{\tau-1}(\cE_{4,0}(u,\tau, \theta_{-u})) \cap N(u)} ( w(e_{u', u}) -w'(e_{u', u})) \right | \right] \\
	= &~\mathbb{E}\Bigg[ \sum_{v \in V \setminus S}  \sum\limits_{u \in V_{S,v}} \sum_{\tau = \tau_1(u)}^{\tau_2(u)-1}  \abs{ \sum_{e \in E_{\tau}(u)} (w(e)-w'(e))}  \Bigg] \,,
\end{align*}
where the last equality comes from the definition of $\tau_1(u)$,$\tau_2(u)$ and $E_\tau(u)$ under weight vector $w$ and the expectation is taken over the randomness of the thresholds on nodes, specifically the value of $\tau_1(u)$,$\tau_2(u)$ and $E_\tau(u)$ for each time step $\tau$. Thus we get the desired result.

\subsection{Proof of the Regret}
\label{app:ltlinucb proof}

The key Theorem \ref{theorem:TPM} describes the difference of the influence spread under two weight vectors in terms of the (expected) weight differences of some edge sets, which coincides with the possible observations under LT model. So this theorem justifies why we distill the information and design the updates of the algorithm $\ltlinucb$ in this way. Next lemma further states the rationality explicitly. Recall that $w$ is the (unknown) true weight vector.


\begin{lemma}\label{lem:combine TPM with alg}

Let $S, w'$ be the seed set and the weight vector output at line \ref{alg:linucb:compute seedset} of $\pairoracle$ in a round $t$. Then for each fixed threshold $\theta \in [0,1]^n$,
\begin{align*}
\sum_{\tau = \tau_1(u)}^{\tau_2(u)-1}\abs{ \sum_{e \in E_{\tau}(u)} (w'(e)-w(e))} \leq  D\cdot \EE{\abs{A_u^\top(w_u'-w_u)}},
\end{align*}
where the definitions of $\tau_1(u),\tau_2(u)$ and $E_\tau(u)$ are defined under weight vector $w$, 
$A_u$ is the value of $A_{t,u}$ updated in lines~\ref{alg:linkbeginupdateAu}--\ref{alg:linefinishupdateAu} in round $t$, which is the distilled edge set chosen by $\ltlinucb$ to update for node $u$,
and the expectation is taken over the randomness of $\tau$ (line \ref{alg:linucb:choosetau}) in determining $A_u$ when $u$ has active in-neighbors.
\end{lemma}

\begin{proof}

Let $D_u = \tau_2(u)-\tau_1(u)$. If $u$ has active in-neighbors, then according to line \ref{alg:linucb:update start}-\ref{alg:linucb:update end} of the Algorithm \ref{alg:LinUCB}, $A_u=\chi( E_{\tau}(u))$ where $\tau = \tau_1(u),\ldots,\tau_2(u)-1$ with probability $1/D_u$ respectively. Thus, 
\begin{align*}
	\EE{\abs{A_u^\top(w_u'-w_u)}} &= \frac{1}{D_u} \sum_{\tau = \tau_1(u)}^{\tau_2(u)-1} \abs{\chi( E_{\tau}(u))^\top (w_u'- w_u)}\\
	&= \frac{1}{D_u} \bracket{ \sum_{\tau = \tau_1(u)}^{\tau_2(u)-1} \abs{\sum\limits_{e \in E_{\tau}(u)} (w'(e)-w(e)) } } \,.
\end{align*}

Since the diffusion process lasts for at most $D$ steps, it is straightforward that $D_u\le D$, thus we get the inequality holds. 

If $u$ has no active in-neighbors, then by definition the values of both LHS and RHS are $0$, thus the inequality still holds. 
\end{proof}

Now we are ready to prove the regret bound.

\begin{proof}[Proof of Theorem \ref{main theorem}]
Define the failure event
\begin{equation}
\cF = \set{\exists t \le T,v \in V: \norm{w_v - \hat{w}_{t,v}}_{M_{t,v}} > \rho_{t,v}}
\end{equation}
to represent the true weight vector $w_v$ does not lie in the confidence ellipsoid $\cC_{t,v}$ for some round $t$ and node $v$. Then by Lemma \ref{lem:ltlinucb:confidence ellipsoid}, when $\rho_{t,v} = \rho_{t}= \sqrt{n \log(1+tn)+2\log\frac{1}{\delta}} + \sqrt{n}$, $\cF^c$ holds with probability at least $1-n\delta$. Next we bound the regret conditioned on the event $\cF^c$.

Recall that $\pairoracle$ is an $(\alpha,\beta)$-approximation oracle, adopted in $\ltlinucb$. Then the $(\alpha,\beta)$-scaled regret of round $t$ satisfies
\begin{align*}
	\EE{R_t} = \EE{\alpha\beta \cdot \opt_w - r(S_t, w)} &\le \EE{\alpha\beta \cdot r(S^{\popt}_{\cC_t}, w^{\popt}_{\cC_t}) - r(S_t, w)} \\
	&\le \EE{r(S_t, w_t) - r(S_t, w)},
\end{align*}
where the last inequality is by the property that $\pairoracle$ is $(\alpha,\beta)$-approximation, and the expectation is over the randomness of the oracle and the randomness in the influence spread.

Then by Theorem \ref{theorem:TPM} and Lemma \ref{lem:combine TPM with alg},
\begin{align*}
\EE{r(S_t, w_t) - r(S_t, w)} &\le D\cdot \ \EE{\sum_{v \in V\setminus S_t}  \sum_{u \in V_{S_t,v}} \abs{A_{t,u}^\top (w_{t,u}-w_u)} } \\
&\le D\cdot \ \EE{\sum_{v \in V\setminus S_t}  \sum_{u \in V_{S_t,v}}  \norm{A_{t,u}}_{M_{t,u}^{-1}} \norm{w_{t,u}-w_u}_{M_{t,u}}} \\
&\leq D\cdot \ \EE{\sum_{v \in V\setminus S_t}  \sum_{u \in V_{S_t,v}}   2\rho_t \norm{A_{t,u}}_{M_{t,u}^{-1}} }, 
\end{align*}
since $w_{t,u}, w_u$ are both in the confidence set. Thus 
\begin{align*}
	R(T) &= \EE{\sum_{t=1}^T R_t} \le 2\rho_T D \cdot \ \EE{ \sum_{t=1}^T \sum_{v \in V\setminus S_t}  \sum_{u \in V_{S_t,v}}  \norm{A_{t,u}}_{M_{t,u}^{-1}} }\\
	&\le 2 \rho_T D \cdot  \EE{ \sqrt{\bracket{\sum_{t=1}^T  \sum\limits_{u \in V} N_{S_t,u}^2}\bracket{\sum_{t=1}^T  \sum\limits_{u \in V} \norm{A_{t,u}}^2_{M_{t,u}^{-1}}}} }\\
	&\le 2 \rho_T D \cdot \ \EE{ \sqrt{T} \gamma(G) \cdot \sqrt{\bracket{ \sum_{t=1}^T \sum_{u \in V} \norm{A_{t,u}}^2_{M_{t,u}^{-1}} } } }
\end{align*}
where the second line is by Cauchy-Schwartz inequality.

Note that $M_{t,u} = M_{t-1,u} +  A_{t,u} A_{t,u}^\top$ and 
\begin{align*}
\det(M_{t,u}) &=  \det\bracket{M_{t-1,u} +  A_{t,u} A_{t,u}^\top} \\
&= \det\bracket{M_{t-1,u}^{1/2}\bracket{I +   M_{t-1,u}^{-1/2} A_{t,u}  A_{t,u}^{\top} M_{t-1,u}^{-1/2} } M_{t-1,u}^{1/2}}\\
&=\det(M_{t-1,u}) \det\bracket{I +   M_{t-1,u}^{-1/2} A_{t,u}  A_{t,u}^{\top} M_{t-1,u}^{-1/2} }\\
&= \det(M_{t-1,u}) \bracket{1 +   \norm{A_{t,u}}_{M_{t-1,u}^{-1}}^2}
\end{align*}
where the last inequality holds because the determinant of a matrix is the product of its eigenvalues and the matrix $I+xx^\top$ has eigenvalues $1$ and $1+\norm{x}_2^2$. And here $\norm{M_{t-1}^{-1/2} A_{t,u}} = \norm{A_{t,u}}_{M_{t-1}^{-1}}$. 
Then
\begin{align}
	\sum_{t=1}^T \sum_{u \in V} \norm{A_{t,u}}^2_{M_{t,u}^{-1}} &\le \sum_{t=1}^T\sum_{u \in V} \frac{n}{\log(1+n)} \cdot \log\bracket{1 +  \norm{A_{t,u}}_{M_{t,u}^{-1}}^2} \label{eq:main:tool} \\
	&\le \sum_{u \in V} \frac{n}{\log(1+n)} \log \frac{\det(M_{T,u})}{\det(I)} \notag \\
	&\le \sum_{u \in V} \frac{n\abs{N(u)}}{\log(1+n)} \log (\trace(M_{T,u})/ \abs{N(u)}) \label{eq:main:detandtrace} \\
	&\le \sum_{u \in V} \frac{n\abs{N(u)}}{\log(1+n)} \log\bracket{1 + \sum_{t=1}^T   \norm{A_{t,u}}_2^2/\abs{N(u)}} \notag\\
	&\le \sum_{u \in V}\frac{n\abs{N(u)}}{\log(1+n)} \log(1+T)\notag \\
	&= \frac{n}{\log (1+n)} \log (1+T) \cdot \sum_{u \in V}\abs{N(u)} \notag \\
	&=  \frac{nm}{\log (1+n)} \log (1+T), \label{eq:main:summ}
\end{align}
where \eqref{eq:main:tool} is by the inequality that $u \le \frac{a}{\log(1+a)} \log(1+u)$ for $u \in [0,a]$ and $ \norm{A_{t,u}}_2^2 \le \abs{N(u)}\le n$; \eqref{eq:main:detandtrace} is by the inequality that $\det(M_{T,u}) \le \bracket{\trace(M_{T,u})/\abs{N(u)}}^{\abs{N(u)}}$ and \eqref{eq:main:summ} holds obviously since the sum of the number of in-neighbors of all nodes is just the number $m$ of edges in the graph.

Therefore the $\alpha \beta$-scaled regret satisfies
\begin{align*}
	R(T) &\le 2\rho_T \gamma(G) D \sqrt{mnT \log(1+T) / \log(1+n)} + n\delta \cdot T(n-k)\\
	&\le C \cdot \gamma(G) D nm^{1/2} \sqrt{T} \log(T),
\end{align*}
for some universal constant $C$.
\end{proof}

\subsection{Discussions}
\label{sec:app:linucb discussion}

\paragraph{Comparisons of regret bounds} We compute the $\gamma(G)$ and $D$ for some special graphs and compare our regret bound with the IMLinUCB algorithm \cite{zhengwen2017nips} and CUCB algorithm \cite{WeiChen2017}, where these two are under IC model and edge-level feedback. The results are listed in Table \ref{table:different graph order} where we use the same examples as in \cite[Figure 1]{zhengwen2017nips}. For general graphs, our $\ltlinucb$ has regret bound $O(\gamma(G)Dnm^{1/2}\sqrt{T}\ln(T)) = O(n^{7/2}m^{1/2}\sqrt{T} \ln(T))$, the IMLinUCB algorithm has regret bound (in the tabular case) $O(C_G m \sqrt{T} \ln(T)) = O(nm^{3/2}\sqrt{T} \ln(T))$ and the CUCB algorithm has regret bound $O(B_G \sqrt{mK'T\ln(T)} ) =O(mn\sqrt{T\ln(T)})$. 
So ours is at most $O(n^{5/2}/m)$ worse than IMLinUCB and $O(n^{5/2}\sqrt{\ln(T)}/\sqrt{m})$ worse than CUCB. Note that the freedom degree of LT model is $O(n)$ as there are only $n$ random variable ($(\theta_v)_{v \in V}$) while the freedom degree of IC model is $O(m)$. Also we assume only node-level feedback is observed while edge-level feedback can be observed in their work on IC model.

\begin{table}[htbp]
\centering
\begin{tabular}{|c|l|l|l|l|l|}
\toprule
\rule{0pt}{12pt} Graphs         & $\displaystyle D$  & $\displaystyle \gamma(G)$       	& $\ltlinucb$ (ours) & IMLinUCB & CUCB  \\ \hline
\rule{0pt}{12pt} bar graph      & \scriptsize $\displaystyle O(1)$  &\scriptsize$\displaystyle O(\sqrt{K})$     	& \scriptsize$\displaystyle O( n^{3/2}\sqrt{KT}\ln(T))$  & \scriptsize $\displaystyle O( n\sqrt{KT}\ln(T))$  & \scriptsize $\displaystyle O( \sqrt{nKT\ln(T)})$  \\ \hline
\rule{0pt}{12pt} star graph     & \scriptsize$\displaystyle O(1)$  & \scriptsize$\displaystyle O(n\sqrt{K})$    	&\scriptsize $\displaystyle O(n^{5/2} \sqrt{KT}\ln(T))$  &\scriptsize $\displaystyle O(n^{2} \sqrt{KT}\ln(T))$  & \scriptsize$\displaystyle O(n^{2} \sqrt{T\ln(T)})$  \\ \hline
\rule{0pt}{12pt} ray graph      &\scriptsize $\displaystyle O(\sqrt{n})$  &\scriptsize $\displaystyle O(n^{5/4}\sqrt{K})$ &\scriptsize $\displaystyle O(n^{13/4} \sqrt{KT}\ln(T))$ &\scriptsize$\displaystyle O(n^{9/4} \sqrt{KT}\ln(T))$ &\scriptsize $\displaystyle O(n^{2} \sqrt{T\ln(T)})$ \\ \hline
\rule{0pt}{12pt} tree graph     &\scriptsize $\displaystyle O(\log n)$  &\scriptsize $\displaystyle O(n^{3/2}) $   		&\scriptsize$\displaystyle O(n^{3}\log n\sqrt{T}\ln(T))$    	&\scriptsize $\displaystyle O(n^{5/2}\sqrt{T}\ln(T))$  &\scriptsize $\displaystyle O(n^{2} \sqrt{T\ln(T)})$ \\ \hline
\rule{0pt}{12pt} grid graph     &\scriptsize $\displaystyle O(n)$  &\scriptsize $\displaystyle O(n^{3/2})$  		&\scriptsize $\displaystyle O(n^{4}\sqrt{T}\ln(T))$  	&\scriptsize$\displaystyle O( n^{5/2}\sqrt{T}\ln(T))$   &\scriptsize $\displaystyle O(n^{2} \sqrt{T\ln(T)})$ \\ \hline
\rule{0pt}{12pt} complete graph &\scriptsize $\displaystyle O(n)$  &\scriptsize $\displaystyle O(n^{3/2})$    		&\scriptsize $\displaystyle O(n^{9/2}\sqrt{T}\ln(T))$ 	&\scriptsize $\displaystyle O(n^{4}\sqrt{T}\ln(T))$  &\scriptsize $\displaystyle O(n^{3} \sqrt{T\ln(T)})$ \\ 
\bottomrule
\end{tabular}
\caption{The values of $\gamma(G),D$ and regret bound comparisons of $\ltlinucb$, IMLinUCB \cite{zhengwen2017nips} and CUCB \cite{WeiChen2017} for special graphs.} 
\label{table:different graph order}
\end{table}
If we represent each edge by a $d$-dimensional feature vector, then we can generalize our $\ltlinucb$ for the large-scale case. The regret bound would become 
\begin{align*}
	O(\rho\gamma(G)D\sqrt{dmn T \log(1+Tn^2/d)}) &= O(\gamma(G)D d \sqrt{mnT} \ln(T)) \\
	&= O(dn^{3} \sqrt{mT} \ln(T))
\end{align*}
where $\rho=O(\sqrt{d\log(1+Tn^2/d) + 2\log(1/\delta)})$. We have used $\norm{A_{t,u}}^2 \le |N(u)|^2$ by assuming the feature vector all have L2-norm at most $1$. The regret bound of IMLinUCB \cite{zhengwen2017nips} under IC model with edge-level feedback is $O(C_G d \sqrt{mT}\ln(T)) = O(dmn\sqrt{T}\ln(T))$, which achieves $\sqrt{m}/n^2$ better order than ours.

\paragraph{GOM property} Our GOM property (Theorem \ref{theorem:TPM}) plays a key role to bound the regret, similar to the TPM condition \cite{WeiChen2017} in the IC model with edge-level feedback. Their proofs \cite{WeiChen2017,zhengwen2017nips} can be simplified by coupling the influence spread under weight $w$ and $w'$ to reduce the proof length significantly (see Appendix \ref{app:simplified proof}).
Under their setting, it is sufficient to prove the key property for monotone case $w \le w'$ since the confidence is estimated for each edge (base arm). The coupling technique can be designed so that the realized graph of $w$ is always a subgraph of $w'$. Then by comparing the connectivity difference in a subgraph, it is easy to derive the desired result.

Situations are different in our setting of node-level feedback. Since only group effect can be observed, we can not guarantee that the representative weight $w'$ is always larger than $w$ (see Section \ref{sec:app_pairoracle} for more discussions). Even though we can prove similar property for monotone $w \le w'$ and hope to generalize it to arbitrary $w,w'$ by leveraging $w \wedge w', w \vee w'$, it does not work. By leveraging $w \wedge w', w \vee w'$, the absolute function would be added to the edge-level (compared with the result formula of Theorem \ref{theorem:TPM}), while we can not observe single edges in group effect. Only the absolute functions on the differences of the weight sum are suitable for node-level feedback.

\section{The Optimization Problem of Weight-Constrained IM}
\label{sec:app_pairoracle}

Recall that we have a confidence ellipsoid $\cC = \set{\cC_v}_{v \in V}$ with
$\cC_v = \set{w_v' \in [0,1]^{|N(v)|}: \norm{w_v' - \hat{w}_v}_{M_v} \le \rho_v}$
and want to consider the optimization problem of weight-constrained influence maximization (WCIM):
\begin{align}
	\argmax_{(S, w'): S \in \cA, w' \in \cC} \ r(S, w')\,. \label{eq:wcim app}
\end{align}
Let $(S^{\popt}_{\cC},w^{\popt}_{\cC})$ be the best solution. We want to find an $(\alpha, \beta)$-approximation oracle that outputs $(S',w')$ with $\PP{r(S',w') \ge \alpha \cdot r(S^{\popt}_{\cC},w^{\popt}_{\cC})} \ge \beta$ for some $\alpha, \beta >0$.

In the following, we first discuss the general difficulty, then give a general solution and later provide efficient methods for some special graph classes.

\subsection{General Difficulties}
\label{sec:app_pairoracle:general difficulty}

\paragraph{The UCB-type method does not directly apply here}

Under the edge-level feedback of the IC model, the learner can update the information of each single edge if it is observed; then the confidence set of the unknown weight vector is just the direct product of the confidence interval over the edges:
\begin{align*}
	\cC = \cC_1 \times \cdots \times \cC_{e} \times \cdots \times \cC_{m}\,,
\end{align*}
where $\cC_{e} = [L(e), U(e)]$ is 1-dimensional confidence interval of weight $w(e)$. Thus if we take the upper bound $U(e)$ of $\cC_e$ for each edge $e$, the resulting vector $U=(U(e))_{e\in E}$ still lie in the confidence set $\cC$ and any weight vector $w' \in \cC$ satisfies $w' \le U$. Since the reward function $r(S,w)$ is monotone increasing in weight vector $w$ (Lemma \ref{lem:monotone and infty norm}), the influence spread of any seed set $S$ under $U$ will be larger than $w'$. Hence $U$ would be the optimal weight vector for the WCIM problem \eqref{eq:wcim app}. Then if we take the output $S_{U}$ from an usual ($\alpha,\beta$)-approximation $\oracle$ for the IM with weight vector $U$, the pair $(S_U, U)$ is an ($\alpha,\beta$)-approximation solution for the problem WCIM. In such derivations, we have described a design of an ($\alpha,\beta$)-approximation $\pairoracle$. This also explains why the designs in \cite{zhengwen2017nips,WeiChen2017} work.

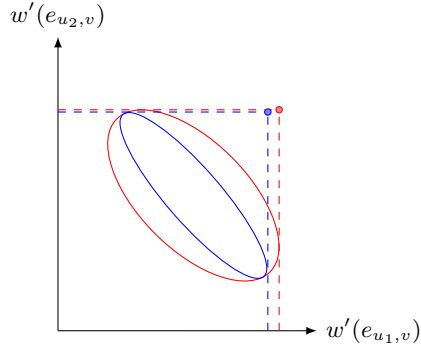
\begin{figure}[hbt!]
    \centering
    \begin{tikzpicture}[scale=0.6,font=\small]

   	\node (x) [] at (4,-3){$w'(e_{u_1,v})$};
	\node (y) [] at (-3,4){$w'(e_{u_2,v})$};

	\draw [-latex](-3,-3)--(x);
	\draw [-latex](-3,-3)--(y);

    \draw [color = red,rotate=135] (0,0) ellipse (2.4cm and 1.2cm);
    \draw [color = blue,rotate=131] (0,0) ellipse (2.38cm and 0.6cm);

    \draw [dashed,color = red](1.9,-3)--(1.9,1.9);
    \draw [dashed,color = red](-3,1.9)--(1.9,1.9);

	\draw[color = red,fill=red!50] (1.9,1.9) circle (2pt);

	\draw [dashed,color = blue](1.65,-3)--(1.65,1.85);
    \draw [dashed,color = blue](-3,1.85)--(1.65,1.85);

	\draw[color = blue,fill=blue!50] (1.65,1.85) circle (2pt);
  
    \end{tikzpicture}
    \caption{An example that upper bound vector fails to lie in the confidence set.}
    \label{fig:confidence set}

\end{figure}

But things are different in the node-level feedback of the LT model. In the node-level feedback, the learner can only observe group effects of edges instead of single edges, so the confidence set is high-dimensional ellipsoid instead of nice cuboid. If we take the upper bounds of each edge, which is equivalent to find the upper confidence bound of the vector $\chi(e)$ with respect to the confidence set $\cC$, the resulting vector might jump out of the confidence set $\cC$. Specifically, since $\ltlinucb$ updates the information of each single node if it has active in-neighbors, the confidence set $\cC$ of the unknown weight vector $w$ is actually the direct product of the confidence set over the nodes:
\begin{align*}
	\cC = \cC_1 \times \cdots \times \cC_{v} \times \cdots \times \cC_{n}\,,
\end{align*}
where $\cC_v$ is a $N(v)$-dimensional confidence set and is related to the edges with ending node $v$. Note that the confidence set $\cC_v$ is different from the above $\cC_e$ and we reuse the notation. For an example of $2$-dimensional case (see Figure \ref{fig:confidence set}), there are two in-neighbors $u_1, u_2$ of $v$ and suppose a confidence ellipse has such a shape (the red ellipse). The vector of largest $w(e'_{u_1}, v), w(e'_{u_2}, v)$ (the red point) is not in the confidence set, and actually is far away from the confidence set. When more observations are collected, the red ellipse may shrink to the blue ellipse, but the vector of largest $w(e'_{u_1}, v), w(e'_{u_2}, v)$ (the blue point) just moves a little and its relative distance to the confidence set is even farther.

\paragraph{Mixed integer optimization problem in bipartite graphs} Consider the special bipartite graphs. The node set $V$ can be divided into $V_1$ and $V_2$ and each edge is from $V_1$ to $V_2$. Without loss of generality, assume $\abs{V_1} \ge K$, then a good solution $S$ must satisfy $S \subset V_1$. So the WCIM problem can be reformulated as
\begin{align}
	&\max_{\alpha, w} \sum_{u \in V_1, v \in N^{\mathrm{out}}(u)} \alpha(u) \ w(u,v) \notag\\
	&s.t. \quad \alpha(u) \in \set{0,1} \text{ for any } u \in V_1 \notag\\
	&\qquad \sum_{u \in V_1} \alpha(u) \le K\\
	&\qquad w(u,v) \in [0,1] \text{ for any } u\in V_1 \text{ and } v \in N^{\mathrm{out}}(u) \notag\\
	&\qquad w(\cdot, v)^{\top} M_v \ w(\cdot, v) \le \rho_v^2 \text{ for any } v \in V_2 \notag
\end{align}
where $w(\cdot, v) \in [0,1]^{\abs{N(v)}}$, $M_v \in \RR^{\abs{N(v)} \times \abs{N(v)}}$ is some positive-definite matrix and $\rho_v$ is some constant. 

This is a mixed integer optimization problem. Even if we relax the constraint of $\alpha(u) \in \set{0,1}$ to $\alpha(u) \in [0,1]$ to make the constraints convex, the objective is bilinear but not convex (or concave), making the problem hard to solve. This mixed integer programming is known to be difficult in the optimization field \cite{so2008unified}. Some techniques of semidefinite programming (SDP) relaxations might be useful. We conjecture the approximation ratio, if solvable, is not constant and is $O(1/\ln(n))$ since there are roughly $n$ constraints for $w$, as also motivated by the greedy method for the problem of max vertex cover. We leave this as interesting future work.

If we write the vector $\alpha$ in a nice vector form, we can see the problem is a special maximum inner product \cite{shrivastava2014MIP1,ram2012MIP2,shen2015MIP3}. This is an interesting direction but there are still many cases unexplored.

\subsection{$\epsilon$-net Method}

The usual oracle for IM problem is to compute the seed set for a given weight vector. Now the confidence set $\cC$ is a continuous set. A method is to discretize it. We can first find an $\epsilon$-net cover, compute the seed set by any usual oracle for each representative, and select the best pair. The complete method is provided in Algorithm \ref{alg:epsilon_net}. Recall that an $\epsilon$-net $\cN$ for a set $\cC$ is for any $w' \in \cC$, there exists a $\pi(w') \in \cN$ such that $\norm{w'- \pi(w')}_2 \le \epsilon$. The minimal size of possible $\cN$ is denoted as $N_{\cC, \epsilon}$.

\begin{algorithm}[tbh!]
\caption{$\epsilon$-net $\pairoracle$}\label{alg:epsilon_net}
\begin{algorithmic}[1]
\STATE \textbf{Input:} Confidence ellipsoid $\cC$; offline IM $\oracle$; seed set cardinality $K$; parameter $\epsilon$
\STATE Find an optimal $\epsilon$-net $\cN$ for $\cC$ with size $N_{\cC, \epsilon}$
\FOR {$\pi \in \cN$}
\STATE Compute the seed set $S_{\pi}$ and $r(S_{\pi}, \pi)$ by $\oracle$ with $\pi$ and $K$
\ENDFOR
\STATE \textbf{Output:} $(S', w') = \argmax_{(S_{\pi}, \pi): \pi \in \cN} \ r(S_{\pi}, \pi)$
\end{algorithmic}
\end{algorithm}

Then we have the following approximation guarantee for the $\epsilon$-net method.

\begin{lemma}\label{lem:approximation}
The Algorithm \ref{alg:epsilon_net} runs with confidence ellipsoid $\cC$, seed set cardinality $K$, parameter $\epsilon$ and an ($\alpha',\beta'$)-approximation $\oracle$. Then its output satisfies
\begin{align*}
	\PP{ r(S', w') \ge \alpha \cdot r(S^{\popt}_{\cC}, w^{\popt}_{\cC})} \ge \beta\,,
\end{align*}
where $\alpha = \alpha' \left(1-\frac{mn \cdot\epsilon}{K}\right)$ and $\beta=\beta'$. Thus the $\epsilon$-net $\pairoracle$ (Algorithm \ref{alg:epsilon_net}) is ($\alpha,\beta$)-approximation.
\end{lemma}

\begin{proof}
For any $w' \in \cC$, let $\pi(w') \in \cN$ be its representative such that $\norm{w'-\pi(w')}_2 \le \epsilon$. Let $S_{w'}^*$ denote the output of $\oracle$ with input $w'$, then $\PP{r(S_{w'}^*, w') \ge \alpha' \cdot \opt_{w'}} \ge \beta'$. Thus
\begin{align*}
r\bracket{S^{\popt}_{\cC},w^{\popt}_{\cC}} &\leq r\bracket{S^{\popt}_{\cC},\pi\bracket{w^{\popt}_{\cC}}} + mn\cdot \epsilon \\
&\leq \opt_{\pi\bracket{w^{\popt}_{\cC}}} + mn\cdot \epsilon \\
&\leq \frac{1}{\alpha'} r\bracket{S^{*}_{\pi\bracket{w^{\popt}_{\cC}}},\pi\bracket{w^{\popt}_{\cC}}} + mn\cdot \epsilon \\
&\leq \frac{1}{\alpha'}r(S',w') +  mn\cdot \epsilon \,,
\end{align*}
where the first inequality is by Lipschitz continuity of $r$ (Lemma \ref{lem:lipschitz}) and $\norm{w^{\popt}_{\cC}-\pi(w^{\popt}_{\cC})}_2 \le \epsilon$, the third inequality is by the definition of $\oracle$ and holds with probability at least $\beta'$ and the last inequality is by the rule of Algorithm \ref{alg:epsilon_net}.

Hence with probability at least $\beta'$, 
\begin{align*}
	r(S',w') \geq \alpha' \cdot \bracket{r(S^{\popt}_{\cC},w^{\popt}_{\cC}) - mn \cdot\epsilon} \ge \alpha'\left(1-\frac{mn \cdot\epsilon}{K}\right)\cdot r(S^{\popt}_{\cC},w^{\popt}_{\cC})\,,
\end{align*}
where the second inequality is by $r(S^{\popt}_{\cC},w^{\popt}_{\cC}) \ge K$.
\end{proof}

The minimal size $N_{\cC,\epsilon}$ of the $\epsilon$-net for the $m$-dimensional ellipsoid $\cC$ has order $\Theta((1/\epsilon)^m)$, which is exponential in $\epsilon$. So this method, though accurate, is not very efficient.

\subsection{Graphs with In-degree at Most 1}
\label{sec:case indegree 1}

We discuss the method to solve the case of graphs that any node has at most one incoming edge. This includes examples in Figure \ref{fig:special graph 1}.
For such graphs, the node-level feedback is actually edge-level feedback. More specifically, our $\ltlinucb$ will update the information of each single edge if its start node is active. Thus the confidence set $\cC$ is the direct product of the confidence intervals of each edge, similar to IC model with edge-level feedback.

\begin{figure}[hbt!]
\centering
    \begin{tikzpicture}[scale=0.57,font=\small,every node/.style={node distance = 1.2cm}]

    \node (a1) [circle,thick,draw=black] at (-24,0) {} ;
    \node (a2) [circle,thick,draw=black,right of=a1] {} ;
    \node (a3) [circle,thick,draw=black,above of=a1,node distance = 1cm]  {} ;
	\node (a4) [circle,thick,draw=black,right of=a3]  {} ;
    \node (a5) [circle,thick,draw=black,above of=a3,node distance = 1cm]  {} ;
    \node (a6) [circle,thick,draw=black,right of=a5]  {} ;
    \node at (-23,-1) {(a)};

    \draw [thick](a1)--(a2);
    \draw [thick](a3)--(a4);
    \draw [thick](a5)--(a6);

 	\node (b1) [circle,thick,draw=black] at (-20.5,0) {} ;
    \node (b2) [circle,thick,draw=black,right of=b1] {} ;
    \node (b3) [circle,thick,draw=black,right of=b2] {} ;
	\node at (-18.3,-1)  {(b)};

    \draw [-latex,thick](b1)--(b2);
    \draw [-latex,thick](b2)--(b3);

    \node (c1) [circle,thick,draw=black] at (-14.6,0) {} ;
    \node (c2) [circle,thick,draw=black,right of=c1,node distance = 0.8cm] {} ;
    \node (c3) [circle,thick,draw=black,right of=c2,node distance = 0.8cm] {} ;
    \node (c4) [circle,thick,draw=black,right of=c3,node distance = 0.8cm] {} ;
    \node (c5) [circle,thick,draw=black,right of=c4,node distance = 0.8cm] {} ;
    \node (c6) [circle,thick,draw=black,above of=c2,node distance = 1cm] {} ;
    \node (c7) [circle,thick,draw=black,above of=c4,node distance = 1cm] {} ;
    \node (c8) [circle,thick,draw=black,above of=c3,node distance = 2cm] {} ;
	\node at (-12.1,-1)  {(c)};

    \draw [-latex,thick](c8)--(c6);
    \draw [-latex,thick](c8)--(c7);
    \draw [-latex,thick](c6)--(c1);
	\draw [-latex,thick](c6)--(c2);
	\draw [-latex,thick](c7)--(c3);
	\draw [-latex,thick](c7)--(c4);
	\draw [-latex,thick](c7)--(c5);

	\node (d1) [circle,thick,draw=black] at (-6,0) {} ;
    \node (d2) [circle,thick,draw=black,above of=d1,node distance = 1cm] {} ;
    \node (d3) [circle,thick,draw=black,right of=d2,node distance = 1cm] {} ;
    \node (d4) [circle,thick,draw=black,left of=d2,node distance = 1cm] {} ;
    \node (d5) [circle,thick,draw=black,above of=d2,node distance = 1cm] {} ;
    \node at (-6,-1)  {(d)};

    \draw [-latex,thick](d2)--(d1);
    \draw [-latex,thick](d2)--(d3);
    \draw [-latex,thick](d2)--(d4);
    \draw [-latex,thick](d2)--(d5);

    \node (e1) [circle,thick,draw=black] at (0,0) {} ;
    \node (e2) [circle,thick,draw=black,above of=e1,node distance = 0.68cm]  {} ;
    \node (e3) [circle,thick,draw=black,above of=e2,node distance = 0.68cm]  {} ;
    \node (e4) [circle,thick,draw=black,above of=e3,node distance = 0.68cm]  {} ;

    \node (e5) [circle,thick,draw=black] at (-2,0.04) {} ;
    \node (e6) [circle,thick,draw=black,above of=e5,node distance = 0.68cm]  {} ;
	\node (e7) [circle,thick,draw=black,above of=e6,node distance = 0.68cm]  {} ;

    \node at (-1,-1) {(e)};

    \draw [-latex,thick](e5)--(e1);
    \draw [-latex,thick](e6)--(e2);
    \draw [-latex,thick](e7)--(e3);
    \draw [-latex,thick](e7)--(e4);

    \end{tikzpicture}

    \caption{Examples of graphs with in-degree at most $1$. (a) bar graph. (b) chain graph. (c) out-arborescence graph. (d) out-star graph. (e) certain bipartite graph. Each undirected edge represents a pair of directed edges pointing to opposite directions.}
    \label{fig:special graph 1}

\end{figure}
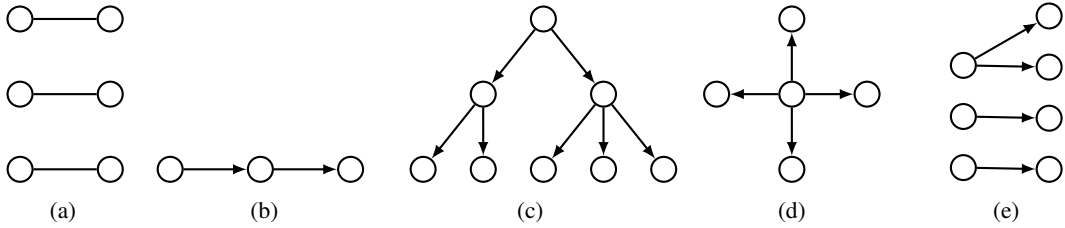

Hence we can just perform as \cite{WeiChen2016,zhengwen2017nips}. As mentioned above, we first take the upper bound for each edge and formulate $U$, then use an ($\alpha,\beta$)-approximation $\oracle$ to compute $S_{U}$ such that
\begin{align}
\PP{r(S_{U},U) \ge \alpha \cdot r(S^{\popt}_{\cC}, w^{\popt}_{\cC})} \ge \beta \,.
\end{align}
So we get an efficient $(\alpha,\beta)$-approximation $\pairoracle$ for these special graphs.

\subsection{Bipartite Graphs}

We consider the special case of bipartite graphs here where there are two node sets $V_1$ and $V_2$ and each edge is from $V_1$ to $V_2$ (see Figure \ref{fig: bipartite graph} for examples). This is a popular influence spread formulation for one step and is a generalization of vertex cover. 

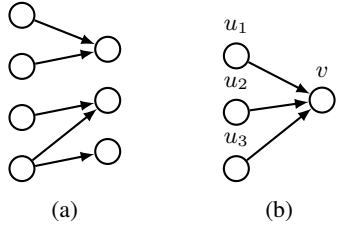
\begin{figure}[hbt!]
\centering
    \begin{tikzpicture}[scale=0.57,font=\small,every node/.style={node distance = 1.2cm}]
    \node (e1) [circle,thick,draw=black] at (-2,0) {} ;
    \node (e2) [circle,thick,draw=black,above of=e1,node distance = 0.68cm]  {} ;
    \node (e3) [circle,thick,draw=black,above of=e2,node distance = 0.68cm]  {} ;
    \node (e4) [circle,thick,draw=black,above of=e3,node distance = 0.68cm]  {} ;

    \node (e5) [circle,thick,draw=black] at (0,0.4) {} ;
    \node (e6) [circle,thick,draw=black,above of=e5,node distance = 0.68cm]  {} ;
	\node (e7) [circle,thick,draw=black,above of=e6,node distance = 0.68cm]  {} ;

    \node at (-1,-1) {(a)};

    \draw [-latex,thick](e1)--(e5);
    \draw [-latex,thick](e1)--(e6);
    
    \draw [-latex,thick](e2)--(e6);

    \draw [-latex,thick](e3)--(e7);
    \draw [-latex,thick](e4)--(e7);

    \node (b1) [circle,thick,draw=black,label=$u_3$] at (3,0) {} ;
    \node (b2) [circle,thick,draw=black,above of=b1,label=$u_2$,node distance = 0.75cm]  {} ;
    \node (b3) [circle,thick,draw=black,above of=b2,label=$u_1$,node distance = 0.75cm]  {} ;

    \node (b6) [circle,thick,draw=black,label=$v$] at (5,1.6)  {} ;

    \node at (4,-1) {(b)};

    \draw [-latex,thick](b1)--(b6);
    \draw [-latex,thick](b2)--(b6);
    
    \draw [-latex,thick](b3)--(b6);

    \end{tikzpicture}

    \caption{Examples of special bipartite graphs. (a) bipartite graph with in-degree at most $2$. (b) bipartite graph with in-degree $3$.}
    \label{fig: bipartite graph}

\end{figure}

Recall that the objective is to solve
\begin{align*}
	\max_{S \in \cA, w' \in \cC} \ r(S, w') = \max_{S \in \cA} \max_{w' \in \cC} r(S, w')\,.
\end{align*}
Note that $r(S, w) = \sum_{u \in S, v \in N^{\mathrm{out}}(u)} w(e_{u,v})$ is linear in $w$ for the bipartite graphs. Let
\begin{align*}
	r(S) := \max_{w' \in \cC} r(S, w') = \max_{w' \in \cC} \sum_{u \in S, v \in N^{\mathrm{out}}(u)} w(e_{u,v})\,.
\end{align*}
Recall that $\cC = \set{w' \in \RR^m : \norm{w' - \hat{w}}_{M} \leq \rho}$ for some positive-definite matrix $M$ and a constant $\rho \ge 0$. So the computation of $r(S)$ is quadratic constrained linear programming and can be solved efficiently. We have the following properties for $r(S)$. The first one is about its monotonicity for any graph.

\begin{lemma}
\label{lem:r_monoton}
For any graph, given the confidence set $\cC$, the function $r(S) = \max_{w' \in \cC} r(S,w')$ is monotone increasing in $S$. That is, $r(S) \le r(S')$ if $S \subseteq S'$.
\end{lemma}
\begin{proof}
Let $r(S) = r(S, w_S), r(S') = r(S', w_{S'})$. Then
\begin{align*}
r(S) = r(S,w_S) \le r(S', w_S) \le r(S', w_{S'}) = r(S')\,.
\end{align*}
\end{proof}

The next one states the submodularity of $r(S)$ for bipartite graphs with in-degree at most $2$ (for example Figure  \ref{fig: bipartite graph}(a)).

\begin{lemma}
\label{lem:bipaetite_submodular}
In bipartite graphs with in-degree at most $2$, the function $r(S) = \max_{w' \in \cC} r(S,w')$ satisfies submodularity. That is, for arbitrary set $S \subseteq S'$ and node $u \notin S'$, there is 
\begin{align}
r(S \cup \{u\}) - r(S) \ge r(S' \cup \{u\}) - r(S)\,.
\end{align}
\end{lemma}
\begin{proof}
As discussed in Section \ref{sec:app_pairoracle:general difficulty}, the confidence set $\cC$ for unknown weight vector $w$ is actually the direct product of confidence set over nodes, that is
\begin{align*}
	\cC = \Pi_{v \in V_2} \ \cC_{v} \,,
\end{align*}
where $\cC_v$ is the confidence set for in-coming edges of node $v$. The edges in each $\cC_v$ are disjoint with each other, so
\begin{align*}
	r(S) = \max_{w' \in \cC} \sum_{u \in S, v \in N^{\mathrm{out}}(u)} w'(e_{u,v}) &= \max_{w' \in \cC} \sum_{v \in V_2} \sum_{u \in S, u \in N^{\mathrm{in}}(v)} w'(e_{u,v}) \\
	&= \sum_{v \in V_2} \max_{w_v' \in \cC_v} \sum_{u \in S, u \in N^{\mathrm{in}}(v)} w_v'(e_{u,v})\,,
\end{align*}
where $w_v' = w'(e_{u,v})_{u \in N^{\mathrm{in}}(v)}$.
So to maximize over $\cC$, it suffices to maximize the weights of incoming edges for each $v \in V_2$. 

Since each node has at most two in-coming edges, if $e_{u, v} \in E$ for some $v$ then it must hold that there is at most one in-neighbor of $v$ from $S'$. For node $v$ such that $e_{u, v} \in E$ but there is no edge from $S'$ to $v$, the contribution of $v$'s part to $S, S'$ are the same.

For node $v$ such that $e_{u, v}, e_{u', v} \in E$ for some $u' \in S' \setminus S$, it suffices to prove that
\begin{align*}
	\max_{w_v' \in \cC_v} w_v'(e_{u, v}) \ge \max_{w_v' \in \cC_v} \set{w_v'(e_{u, v}) + w_v'(e_{u', v})} - \max_{w_v' \in \cC_v} w_v'(e_{u', v}) \,,
\end{align*}
which is obviously true.

For node $v$ such that $e_{u, v}, e_{u', v} \in E$ for some $u' \in S \subseteq S'$, the contribution of $v$'s part to $S, S'$ are the same.
\end{proof}

\begin{algorithm}[th]{}
\caption{Greedy $\pairoracle$}
\label{alg:oracl:greedy}
\begin{algorithmic}[1]
\STATE \textbf{Input:} Graph $G=(V,E)$, confidence set $\cC$, seed set cardinality $K$
\STATE \textbf{Initialize:} $S = \emptyset$
\FOR{$i \in [K]$}
    \STATE $v = \argmax_{u \in V\setminus S} \ r(S\cup\{u\}) - r(S)$
    \STATE $S = S \cup \{v\}$
\ENDFOR
\STATE Output $S$
\end{algorithmic}
\end{algorithm}

With the submodularity property, we can get the approximation result by designing a greedy policy (Algorithm \ref{alg:oracl:greedy}).

\begin{lemma}
\label{lem:greedyoracle_barpartite}
Recall that $S_{\cC}^{\popt} = \argmax_{S \in \cA}\max_{w' \in \cC} r(S,w') = \argmax_{S \in \cA} r(S)$ is the optimal seed set given confidence set $\cC$. Let $S'$ be the solution returned by Greedy $\pairoracle$ (Algorithm \ref{alg:oracl:greedy}). Then for bipartite graphs with in-degree at most $2$,
\begin{align}
	r(S') \ge \bracket{1-\frac{1}{e}} r(S_{\cC}^{\popt}) \,.
\end{align}
\end{lemma}
The proof is a direct application of \cite[Theorem 2.1]{Kempe2003} by noting that the function $r(\cdot)$ satisfies monotonicity (Lemma \ref{lem:r_monoton}) and submodularity (Lemma \ref{lem:bipaetite_submodular}) in such graphs.

\paragraph{A counterexample of in-degree $3$} Here we show an example of bipartite graphs with in-degree $3$ but the $r(\cdot)$ does not have the submodularity property.

Let $V_1 = \set{u_1, u_2, u_3}, \abs{V_2} = 1$ and there are only $3$ edges (see Figure \ref{fig: bipartite graph}(b) for example). The confidence set $\cC = \set{w' \in \RR^3 : \norm{w'}_M \le 1}$ with 
$$M = \left[
\begin{array}{ccc} 
2 & 1 & 0 \\
1 & 3 & 1\\
0 & 1 & 2
\end{array} 
\right]\,.$$
Note $M = I + (1, 1, 0)^\top (1, 1, 0) + (0, 1, 1)^\top (0, 1, 1) $ can happen for our algorithm $\ltlinucb$. We solve the optimization problem and get $r(\set{u_1, u_2}) \approx 0.791, r(\set{u_2, u_3}) \approx 0.791, r(\set{u_2}) \approx 0.707, r(\set{u_1, u_2, u_3}) \approx 1.000 $. Thus let $u = u_1, S = \set{u_2}, S'=\set{u_2,u_3}$, we have 
\begin{align*}
r(S\cup \set{u}) - r(S) < 0.09 < 0.2 < r(S'\cup \set{u}) - r(S')\,,
\end{align*}
which violates the definition of submodularity.

\subsection{Directed Acyclic Graphs}

Recall that $r(S) = \max_{w' \in \cC} r(S,w')$ and $S_{\cC}^{\popt} = \argmax_{S \in \cA} r(S)$ is the optimal solution. Let $S'$ be the output of Greedy $\pairoracle$ (Algorithm \ref{alg:oracl:greedy}). Then we have the following $1/K$-approximation result.

\begin{lemma}
\label{lem:greedyoracle_general}
For general graphs, suppose we can compute $r(S)$ for any $S$. Then
\begin{align}
r(S') \ge \frac{1}{K} \cdot r(S_{\cC}^{\popt}) \,.
\end{align} 
\end{lemma}
\begin{proof}
Denote $S_{\cC}^{\popt} = \set{s_1^*,s_2^*,...,s_K^*}$. 
Assume the Greedy $\pairoracle$ first chooses $s'$. Then $s' \in S'$ and $s' = \argmax_{u \in V} r\bracket{\{u\}}$ or equivalently $r(\{s'\}) \ge r\bracket{\{u\}}$ for any $u \in V$. By monotonicity of $r$ (Lemma \ref{lem:r_monoton}),
\begin{align*}
r(S') \ge r(\{s'\}) \ge \frac{1}{K}\cdot \bracket{r(\{s_1^*\})+r(\{s_2^*\})+...+r(\{s_K^*\})} \,.
\end{align*}

It suffices to prove that $r$ satisfies the subadditivity. It is well known that the reward function $r(\cdot, w')$ satisfies submodularity in LT model \cite{Kempe2003}. Then for any $S \subseteq S''$,
\begin{align*}
r(S, w') + r(S'', w') \ge r(S\cup S'', w') + r(S \cap S'', w') \ge r(S\cup S'', w') \,.
\end{align*}
Recall that $r(S_{\cC}^{\popt}) = r(S_{\cC}^{\popt}, w_{\cC}^{\popt})$. Then
\begin{align*}
&r(\{s_1^*\}) + r(\{s_2^*\}) + \ldots + r(\{s_K^*\}) \\
&\qquad \ge r(\{s_1^*\}, w_{\cC}^{\popt}) + r(\{s_2^*\}, w_{\cC}^{\popt}) + \ldots + r(\{s_K^*\},w_{\cC}^{\popt}) \\
&\qquad \ge r(S_{\cC}^{\popt}, w_{\cC}^{\popt}) = r(S_{\cC}^{\popt})
\end{align*}
and the result follows.
\end{proof}

Next we show that for directed acyclic graphs (DAGs), there is an efficient method to compute $r(S) = \max_{w' \in \cC} r(S, w')$. 

\begin{algorithm}[h]{}
\caption{Compute $r(S)$ in DAGs}
\label{alg:DAGs solve r}
\begin{algorithmic}[1]
\STATE \textbf{Input:} DAG $G=(V,E)$; seed set $S$; \\
\qquad \quad the set of confidence ellipsoids $(\cC_v)_{v \in V}$ with $\cC_v = \set{w_v' : (w_v')^\top M_v w_v' \le \rho_v^2}$
\STATE \textbf{Initialize:} Delete all in-edges to nodes in $S \subseteq V$
\STATE Use topological ranking to form `layers' of nodes $L_0, \ldots, L_\ell, \ldots, L_{n-1}$ satisfies any edge $e \in E$ points from $L_i$ to $L_j$ for some $i < j$
\STATE $r_S^u = 1$ for $u \in S$; $r_S^u = 0$ for $u \in L_0 \setminus S$
\FOR{$\ell=1,2,\ldots$}
    \FOR{$v \in L_{\ell}$}
    \STATE Solve $r_S^v = \max_{w' \in \cC_v} \sum_{u \in N(v)} r_{S}^{u} \cdot w'(e_{u, v})$ 
	\ENDFOR
\ENDFOR
\STATE \textbf{Output:} $r(S) = \sum_{v \in V} r_{S}^{v}$
\end{algorithmic}
\end{algorithm}

For seed set $S$, delete all in-coming edges to $S$. Take all nodes with in-degree $0$ and form a set $L_0 \supseteq S$. Then consider the reduced subgraph for remaining nodes $V \setminus L_0$, take all nodes in the subgraph with in-degree $0$ and form a set $L_1$. 
Note subgraphs of DAGs are still DAGs and in DAGs there are nodes with in-degree $0$, otherwise we could find a cycle by adaptively adding in-neighbors. Then the procedure can continue until no node is left. Such process is just topological ranking to form `layers' of nodes. For any node $u \in L_{\ell}$, its incoming edges are all from previous layers (except seed nodes), or equivalently nodes in $L_0 \cup L_1 \cup \cdots \cup L_{\ell-1}$. There are at most $n$ layers.

Let $E_{\ell}'$ to denote the edges that has end node in layer $\ell$ and $E_{\ell:\ell'}' = E_{\ell}' \cup E_{\ell+1}' \cup \cdots \cup E_{\ell'}'$. Then $E_{\ell}' \bigcap E_{\ell'}' = \emptyset$ if $\ell \neq \ell'$.

Let $r_{S}^{u}(w')$ be the probability that node $u$ will be influenced under the weight vector $w'$ when the seed set is $S$ and $r_{S}^{u} = \max_{w' \in \cC} r_{S}^{u}(w')$. 
For seed node $u \in S$, it is activated with probability $1$, or $r_{S}^{u}(w') \equiv 1$. For node $u \in L_0 \setminus S$, there is no directed path connecting from seed node $S$, so its activation probability is always $0$, or $r_{S}^{u}(w') \equiv 0$. So we have computed $r_{S}^{u}$ for $u \in L_0$.

Let $\ell = 1$.
For node $u \in L_{\ell}$, its incoming edges all come from former layers $< \ell$. Note that $r_S^u$ has been defined for any layer $<\ell$ and $r(S, w')$ can be decomposed as
\begin{align}
	\sum_{v \in L_\ell} \bracket{ \sum_{u \in N^{\mathrm{in}}(v)} r_{S}^{u} \cdot w'(e_{u, v}) } \cdot f_v(w', E_{\ell+1:n}')\,, \label{eq:decomp by layer}
\end{align}
where $f_v(w', E_{\ell+1:n}')$ is the expected influenced nodes by node $v$ for \textit{later layers} and it only relates with the edges ending in later layers.
Recall that the constraints are added to the edges with the same ending node. The edge $e_{u,v}$ for $v \in L_{\ell}$ ends in $L_{\ell}$, so it is independent with $E_{\ell'}'$ for $\ell' > \ell$. Note that $f_v(w', E_{\ell+1:n}') \ge 1 > 0$ since node $v$ at least influences itself. So to maximize $r(S, w')$ over $w' \in \cC$, the weights related with edges in $E_{\ell}'$ can be maximized separately. Specifically, we can solve the maximization problem for each $v \in L_{\ell}$:
\begin{align}
	&\max \sum_{u \in N^{\mathrm{in}}(v)} r_{S}^{u} \cdot w'(e_{u, v}) \label{eq:dag optimization} \\ 
	&s.t. \quad (w_v')^\top M_v w_v' \le \rho_v^2 \notag
\end{align}
where the $w_v' = (w'(e_{u, v}))_{u \in N(v)}$, $M_v$ is some positive-definite matrix and $\rho_v$ is some constant. This optimization problem is linear programming with quadratic convex constraints and can be solved efficiently. The resulting maximum value is actually $r_{S}^{v}$. So we can compute $r_S$ for layer $\ell$. Then we can compute $r(S) = \sum_{v \in V} r_S^v$ by repeating steps \eqref{eq:decomp by layer} \eqref{eq:dag optimization} with induction on $\ell$. The process is presented in Algorithm \ref{alg:DAGs solve r}.

The key point to make this through for DAGs is based on the linearity of LT. Then we can decompose the objective functions to \textit{isolated} parts and use common optimization methods to solve each part step by step.

\section{Analysis of $\oimetc$ Algorithm}
\label{sec:app etc}

We first provide the regret bound of $\oimetc$ Algorithm under both IC and LT models and then give discussions about it.

\subsection{Proof of Theorem \ref{thm:etc}}

Recall that $\hat{w}$ is the empirical estimate of weight vector $w$ (line \ref{algo:etc:compute hat w} of Algorithm \ref{alg:IMETC}) and $\hat{S}$ is the output of the $(\alpha,\beta)$-approximation $\oracle$ under estimated weight vector $\hat{w}$ (line \ref{algo:etc:compute hat S} of Algorithm \ref{alg:IMETC}).
Define event
\begin{align*}
	\cF &= \set{r(\hat{S}, \hat{w}) < \alpha\cdot \opt_{\hat{w}}}\,.
\end{align*}
Then $\PP{\cF} < 1-\beta$ since the $\oracle$ is $(\alpha,\beta)$-approximation.

We first decompose the regret
\begin{align}
R(T) &= \EE{\sum\limits_{t=1}^{T} \bracket{\alpha\beta\cdot \opt_{w} -  r(S_t,w) } } \notag\\
&= \EE{\sum\limits_{t=1}^{nk} \bracket{\alpha\beta\cdot \opt_{w} -  r(S_t,w) } } + \EE{\sum\limits_{t=nk+1}^{T} \bracket{\alpha\beta\cdot \opt_{w} -  r(S_t,w) } } \notag \\
&\le nk \Delta_{\max} + (T-nk)\EE{\alpha\beta\cdot \opt_{w} -  r(\hat{S},w)} \notag \\
&\le nk \Delta_{\max} + (T-nk) \beta \cdot \EE{\alpha \cdot \opt_{w} -  r(\hat{S},w) \middle| \cF^c} \label{eq:etc pf condition}
\end{align}
where the last inequality is by 
\begin{align*}
	\EE{r(\hat{S},w)} = \EE{r(\hat{S},w) \middle| \cF} \PP{\cF} +\EE{r(\hat{S},w) \middle| \cF^c} \PP{\cF^c} \ge \beta \cdot \EE{r(\hat{S},w) \middle\vert \cF^c}\,.
\end{align*}

Note under $\cF^c$,
\begin{align}
	\alpha \cdot \opt_{w} &= \alpha \cdot r(S_w^\opt, w) \notag\\
	&\le \alpha \cdot r(S_w^\opt, \hat{w}) + \alpha \cdot mn \cdot \max_{e \in E} \abs{\hat{w}(e) - w(e)} \notag \\
	&\le \alpha \cdot r(S_{\hat{w}}^\opt, \hat{w}) + \alpha \cdot mn \cdot \max_{e \in E} \abs{\hat{w}(e) - w(e)} \notag \\
	&\le r(\hat{S}, \hat{w}) + \alpha \cdot mn \cdot \max_{e \in E} \abs{\hat{w}(e) - w(e)}\notag \\
	&\le r(\hat{S}, w) + (1+\alpha) \cdot mn \cdot \max_{e \in E} \abs{\hat{w}(e) - w(e)}\,. \label{eq:etc gap}
\end{align}
Then when $\max_{e \in E} \abs{\hat{w}(e) - w(e)} < \frac{\Delta_{\min}}{(1+\alpha)mn} =: \epsilon_0$, $\hat{S} \notin \cS_B$. So the regret becomes
\begin{align}
R(T) &\leq nk\Delta_{\max} + (T-nk)\cdot 2m\exp(-2k\epsilon_0^2)\Delta_{\max} \notag\\
&\leq \bracket{nk + 2mT\exp(-2k\epsilon_0^2)}\Delta_{\max} \notag\\
&= \frac{n\Delta_{\max}}{2\epsilon_0^2} \ln^+\frac{4mT\epsilon_0^2}{n} + \frac{n\Delta_{\max}}{2\epsilon_0^2} \notag
\end{align}
where the first inequality is to bound the complement of the event $\max_{e \in E} \abs{\hat{w}(e) - w(e)} < \epsilon_0$ by the Chernorff-Hoeffding bound (Lemma \ref{lem:Ch-H}), the equality is optimized with $k$ satisfying $\exp(2k\epsilon_0^2) = 4mT\epsilon_0^2/n$ and $\ln^+(x) = \max\set{0, \ln(x)}$.

Therefore taking $k = \max\set{1, \frac{1}{2\epsilon_0^2}\ln\frac{4mT\epsilon_0^2}{n}} = \max\set{1, \frac{2m^2 n^2}{\Delta_{\min}^2}\ln\frac{T \Delta_{\min}^2}{m n^3}}$ together with $R(T) \le T \Delta_{\max}$, the regret satisfies
\begin{align}
R(T) &\leq \min\set{T\Delta_{\max}, n\Delta_{\max} + \frac{2m^2 n^3\Delta_{\max}}{\Delta_{\min}^2}\bracket{1 + \ln^+\frac{T\Delta_{\min}^2}{mn^3}}} \notag\\
&= O\bracket{\frac{m^2 n^3\Delta_{\max}}{\Delta_{\min}^2} \ln(T)}\,.
\end{align}

Next we prove the problem-independent bound. Following \eqref{eq:etc gap} under $\cF^c$, with a suitable $\epsilon$ to be decided later,
\begin{align*}
	\EE{\alpha \cdot \opt_{w} - r(\hat{S}, w)} &\le 2mn \cdot \EE{\max_{e \in E} \abs{\hat{w}(e) - w(e)}}\\
	&\le 2mn \epsilon + 2mn \sum_{s=0}^\infty 2^{s+1} \epsilon \cdot \PP{2^s \epsilon < \max_{e \in E} \abs{\hat{w}(e) - w(e)} \le 2^{s+1} \epsilon}\\
	&\le 2mn \epsilon + 2mn \sum_{s=0}^\infty 2^{s+1} \epsilon \cdot \PP{\exists e \in E, \abs{\hat{w}(e) - w(e)} > 2^s \epsilon}\\
	&\le 2mn \epsilon + 2mn \sum_{s=0}^\infty 2^{s+1} \epsilon \cdot 2m \exp(-2k 2^{2s} \epsilon^2)\\
	&= 2mn \epsilon + \frac{8m^2n}{\sqrt{2k}} \sum_{s=0}^\infty \sqrt{2k}2^s\epsilon \cdot \exp(-(\sqrt{2k}2^s\epsilon)^2)\,.
\end{align*}
Let $X_s:= \sqrt{2k}2^s\epsilon$. Note that the function $f(x) = xe^{-x^2}$ increases in $[0,1/\sqrt{2}]$ and decreases in $[1/\sqrt{2}, \infty)$. Let $s_0$ satisfy 
\begin{align*}
2^{s_0} < \frac{1}{2\sqrt{k}\epsilon} \le 2^{s_0+1}\,,
\end{align*}
or equivalently $X_s < 1/\sqrt{2}$ for $s \le s_0$ and $X_s \ge 1/\sqrt{2}$ for $s \ge s_0+1$. Then we can divide the sum into three parts
\begin{align*}
	\sum_{s=0}^{s_0-1} f(X_s) + \sum_{s=s_0}^{s_0+1} f(X_s) + \sum_{s_0+2}^\infty f(X_s)\,.
\end{align*}
By monotonicity, $\sum_{s=0}^{s_0-1} f(X_s) \le \int_0^{s_0} f(x) \ dx$ and $\sum_{s_0+2}^\infty f(X_s) \le \int_{s_0+1}^{\infty} f(x) \ dx$. 
Thus
\begin{align*}
	\EE{\alpha \cdot \opt_{w} - r(\hat{S}, w)} &\le 2mn \epsilon + \frac{8m^2n}{\sqrt{2k}} \bracket{\int_{0}^{\infty} f(x) \ dx + f(X_{s_0}) + f(X_{s_0+1})}\\
	&\le 2mn \epsilon + \frac{8m^2n}{\sqrt{2k}} \bracket{\frac{1}{2} + 2f(1/\sqrt{2})} \\
	&= 2mn \epsilon + \frac{8m^2n}{\sqrt{2k}} \bracket{\frac{1}{2} + \sqrt{2}\exp(-1/2)} \\
	&\le 2mn \epsilon + 7.69 m^2n / \sqrt{k}\,.
\end{align*}

By substituting it to \eqref{eq:etc pf condition}, the regret is bounded by 
\begin{align}
	R(T) &\le nk\Delta_{\max} + T \cdot \bracket{2mn \epsilon + 7.69 m^2n / \sqrt{k}} \notag\\
	&\le n^2k + T \cdot \bracket{2mn \epsilon + 7.69 m^2n / \sqrt{k}} \notag\\
	&\le 3.9 (mn)^{4/3} T^{2/3} + 1 \le 5(mn)^{4/3} T^{2/3} \notag \\
	&= O\bracket{(mn)^{4/3} T^{2/3}}\,,
\end{align}
where we take $k=3.9 m^{4/3} n^{-2/3} T^{2/3}$ and $\epsilon=1/(2mnT)$.

\subsection{Discussions}
\label{app:etc discussions}

As we mentioned, our $\oimetc$ algorithm is model independent and applies to both LT and IC model with node-level feedback.

Recall that for a typical influence spread under the IC model, each edge $e$ is \textit{live} with the associated probability $w(e) \in [0,1]$ and a node is activated if there is a (directed) path connecting from the seed set. For the IC model, there are three types of feedback: (1) bandit feedback, where the learner can only observe the number of influenced nodes; (2) edge-level feedback, where the learner can observe the liveness status of each outgoing edge from the activated nodes; (3) node-level feedback, where the learner can only observe the spread propagation but not individual edge liveness. The bandit feedback presents the least information and is most difficult considering the nonlinearity and complexity of the influence reward function. The edge-level feedback gives the most informative feedback and most previous work study this scheme \cite{WeiChen2016,WeiChen2017, zhengwen2017nips,IMFB2019}.

Since our $\oimetc$ only selects size-$1$ seed set in the exploration phase, so the node-level feedback of the first-step triggering is the same with the edge-level feedback. Thus $\oimetc$ can be applied to both IC and LT model. Though simple, $\oimetc$ is the first model-independent algorithm for OIM\footnote{Note that the work \cite{Model-Independent2017} presents a model-free algorithm for an approximated reward function without approximation ratio while we do not relax the spread objective.}. Furthermore, the computational complexity for $\oimetc$ is really low, as it only calls once of the offline oracle.

As mentioned in related work, the algorithm for the combinatorial partial monitoring \cite{lin2014combinatorialpm} can be applied in OIM for both LT and IC models with node-level feedback. However, the second best solution used in their algorithm could not be directly computed in the offline IM setting. Hence only their the second stop-exploration condition applies and a regret bound of $O(nm^{3/2}T^{2/3}\ln(T))$ is obtained. Our $\oimetc$ is better in $O(\ln(T))$ term and a bit worse in $O((n/\sqrt{m})^{1/3})$ term. Also our $\oimetc$ has a problem-dependent regret bound.

Comparing with $\ltlinucb$ we see that $\oimetc$ only requires the first-step node feedback, not the full diffusion process feedback of $S_{t,0}, S_{t,1}, \ldots, S_{t,\tau}, \ldots$.
Moreover, it only requires the offline oracle to solve the maximization problem using the empirical mean as the fixed weight vector. 
The objective function in this case is known to be monotone and submodular \cite{Kempe2003,Mossel2010local2global}, and thus a greedy algorithm \cite{Kempe2003} or IMM algorithm \cite{IMM2015} could achieve $1-1/e-\varepsilon$ approximation (for any small $\varepsilon > 0$) with probability at least $1-1/n$.
That is, $(\alpha,\beta)$-approximation $\oracle$ with $\alpha= 1-1/e-\varepsilon$ and $\beta = 1-1/n$ has an efficient implementation.
This is also easier than the $\pairoracle$, which has the confidence ellipsoid as the constraint on weight vectors.

\section{Technical Lemmas}
\label{sec:tech lemmas}

\begin{lemma}
\label{lem:Ch-H}
(Chernorff-Hoeffding bound) Let $X_1,X_2,\ldots,X_n$ be independent random variables with common support $[0,1]$. Let $S_n = X_1+X_2+ \cdots +X_n$ and $\mu = \EE{S_n}$. Then for any $\epsilon \geq 0$,
\begin{align*}
	\PP{\abs{S_n - \mu} \geq n \epsilon} \leq 2\exp\bracket{-2n\epsilon^2}\,.
\end{align*}
\end{lemma}

Next is a property of the reward function on the weight vector under the LT model. Note that the similar property also holds for IC model \cite[Lemma 6]{WeiChen2016}.

\begin{lemma}
\label{lem:monotone and infty norm}
Under the LT model, the reward function $r(S,w)$ is monotone increasing in $w$.
And for any seed set $S$ and any two weight vectors $w,w' \in [0,1]^{m}$, there is 
\begin{align}
\label{eq:lipschitz continuity}
\abs{r(S,w) - r(S,w')} \leq mn \cdot \max_{e \in E} \abs{w(e)-w'(e)}\,.
\end{align}
\end{lemma}

\begin{proof}
We first prove the monotonicity. Suppose $w(e) \le w'(e)$ for all $e \in E$. For any fixed thresholds $\theta_v$'s, the instance of influence graph under weight vector $w$ is always a subgraph of $w'$ since any activated node $v$ under $w$ is always activated under $w'$. Thus $r(S, w) \le r(S, w')$.

For \eqref{eq:lipschitz continuity}, it is enough to prove the case $w \le w'$; otherwise we can prove it first for $w \wedge w'$ and $w \vee w'$ and then conclude the result since $r(S, w \wedge w') \le \set{r(S, w), r(S, w')} \le r(S, w \vee w')$.

Now assume $w, w'$ only differ on one edge $e$: $w'(e) > w(e)$ and $w'(e') = w(e')$ for any $e' \neq e$. For any fixed thresholds $\theta_v$'s, consider the two diffusion process under $w, w'$. If the spreads are different, then the starting node that the diffusion processes starts to become different must be the end node of edge $e$. Then this difference would cause at most $n$ nodes differences. Such an event happens when the difference of $w'(e) - w(e)$ contributes to the activation of end node of edge $e$, which has probability at most $w'(e) - w(e)$. Thus $r(S,w') - r(S,w) \le n \cdot \bracket{w'(e) - w(e)}$.


Then for vectors $w \le w'$, we can construct at most $\abs{E} = m$ vector pairs from $w$ to $w'$ with each pair only differing on one edge. By summing them up, we get $r(S,w') - r(S,w) \le mn \cdot \max_{e \in E} \abs{w(e)-w'(e)}$.
\end{proof}

\begin{lemma}\label{lem:lipschitz}
For any seed set $S$ and any two weight vectors $w, w' \in [0,1]^m$, there is
\begin{align*}
	\abs{r(S,w) - r(S,w')} \le mn \cdot \norm{w - w'}_2\,.
\end{align*}
\end{lemma}

\begin{proof}
Lemma \ref{lem:lipschitz} can be concluded directly from Lemma \ref{lem:monotone and infty norm} since it is obvious that $\max_{e \in E} \abs{w(e)-w'(e)} \le \norm{w-w'}_2$.
\end{proof}

\section{A Simplified Proof for the TPM Condition}
\label{app:simplified proof}

We give a simplified proof for the TPM condition under the IC model with edge-level feedback, which corresponds to \cite[Theorem 3]{zhengwen2017nips} and the key equation \cite[Lemma 2, (28)]{WeiChen2017}. For completeness, we also give the theorem statement here, which mainly follow the notations of \cite{zhengwen2017nips}.

$f(S,w,v)$ is the influence probability of seed set $S$ to node $v$ when the mean of the weights is vector $w$. $O(e)$ denotes the event that edge $e$ is observed. Recall that an edge $e$ is relevant with $S, v$ means there exists a path $\ell$ from a seed node $s \in S$ to $v$ such that (1) $e \in \ell$ and (2) $\ell$ does not contain another seed node other than $s$. In the following, we use boldface $\bw$ to represent a random realization of the weight vector.

\begin{theorem} (restated)
For any node $v \notin S$,
\begin{align}
f(S, U, v) - f(S, \bar{w}, v) \le \sum_{e \text{ is relevant with } S, v} \mathbb{E}_{\bar{w}}[\bOne{O(e)} \cdot (U(e) - \bar{w}(e)) \mid S]
\end{align}
\end{theorem}

\begin{proof}
Note that
\begin{align*}
	f(S, U, v) &= \mathbb{E}_{\bw_1 \sim U} \bOne{v \text{ is influenced under } \bw_1} \,,\\
	f(S, \bar{w}, v) &= \mathbb{E}_{\bw_2 \sim \bar{w}} \bOne{v \text{ is influenced under } \bw_2}\,.
\end{align*}
When we compute the difference of these two terms, we do not need to make these two $\bw$ independent. Specifically, for each edge $e$, we can design $\bw_1, \bw_2$ in the following way. Suppose for each edge $e$ we independently draw a uniform random variable $A(e)$ over $[0,1]$, let
\begin{align*}
	\bw_1(e) &= \bw_2(e) = 1, &\text{ if } A(e) \le \bar{w}(e)\,; \\
	\bw_1(e) &= 1, \bw_2(e) = 0, &\text{ if } A(e) \in (\bar{w}(e), U(e)]\,;\\
	\bw_1(e) &= \bw_2(e) = 0, &\text{ if } A(e) > U(e)\,.
\end{align*}
Such a design of $\bw_2$ would introduce a subgraph of $\bw_1$ and the marginal expected means of $\bw_1, \bw_2$ are $U, \bar{w}$ respectively. Then the difference would become much simpler
\begin{equation*}
	f(S, U, v) - f(S, \bar{w}, v) = \mathbb{E}_{\bw_1,\bw_2 \sim A} [f(S, \bw_1, v) - f(S, \bw_2, v)]
\end{equation*}
and $f(S, \bw_1, v) - f(S, \bw_2, v) = 0$ or $1$. 

$f(S, \bw_1, v) - f(S, \bw_2, v) = 1$ means $f(S, \bw_1, v) = 1$ and $f(S, \bw_2, v) = 0$. Thus for any path $\ell$ from $S$ to $v$ in $\bw_1$, there is an edge $e \in \ell$ such that $e \notin \bw_2$. We take first such $e = (u_1, u_2)$, thus the edges on $\ell$ before $e$ are live in $\bw_2$ and the starting node $u_1$ of $e$ is activated under $\bw_2$ without edge $e$. Therefore there is an edge $e = (u_1, u_2)$ on the path from $S$ to $v$ such that
\begin{enumerate}
	\item $u_1$ is activated by $\bw_2$ on the graph without edge $e$\,;
	\item $\bw_1(e) = 1, \bw_2(e) = 0$\,.
\end{enumerate}
Such an edge $e$ is relevant with $S$ and $v$. 
Thus
\begin{align*}
	&\mathbb{E}_{\bw_1,\bw_2 \sim A}[f(S, \bw_1, v) - f(S, \bw_2, v)] \\
	&\qquad \qquad \qquad \le \sum_{e\mbox{ is relevant with } S,v } \mathbb{E}_{\bw_2}[\bOne{ e \text{ is observed under } \bw_2} \cdot (U(e) - \bar{w}(e))]\,.
	\end{align*}
\end{proof}

With the help of this theorem, we can get the same result of TPM conditions in the work \cite{zhengwen2017nips,WeiChen2017}.

\fi

\end{document}